\documentclass[twoside]{article}
\usepackage{aistats2016}
\usepackage[colorlinks = true,urlcolor= blue, linkcolor = red, citecolor= blue]{hyperref}
\usepackage{url}

\usepackage{algorithm}
\usepackage{algorithmic}
\usepackage{amsfonts}
\usepackage{amsmath}
\usepackage[titletoc]{appendix}
\usepackage{bm}
\usepackage{dsfont} 
\usepackage{multicol}
\usepackage{multirow}
\usepackage{booktabs}
\usepackage{arydshln}
\usepackage{gensymb}
\usepackage{wrapfig}
\usepackage{capt-of}
\usepackage{wrapfig}
\usepackage{graphicx} % more modern
\usepackage[update,prepend]{epstopdf}
\usepackage{epsfig}
\usepackage{subfigure} 

\DeclareGraphicsExtensions{.pdf,.eps,.png,.jpg}

\usepackage{natbib}
\begin{document}



\twocolumn[

\aistatstitle{$p$-Markov Gaussian Processes for Scalable and Expressive Online Bayesian Nonparametric Time Series Forecasting}
\aistatsauthor{ Yves-Laurent Kom Samo \\ ylks@robots.ox.ac.uk \And Stephen J. Roberts \\ sjrob@robots.ox.ac.uk }

\aistatsaddress{University of Oxford \And University of Oxford}  ]

\begin{abstract}
In this paper we introduce a novel online time series forecasting model we refer to as the $p$M-GP filter. We show that our model is equivalent to Gaussian process regression, with the advantage that both online forecasting and online learning of the hyper-parameters have a constant (rather than cubic) time complexity and a constant (rather than squared) memory requirement in the number of observations, without resorting to approximations. Moreover, the proposed model is expressive in that the family of covariance functions of the implied latent process, namely the spectral Mat\'{e}rn kernels, have recently been proven to be capable of approximating arbitrarily well any translation-invariant covariance function. The benefit of our approach compared to competing models is demonstrated using experiments on several real-life datasets.
\end{abstract}

\section{INTRODUCTION}
Interest in analysing and forecasting time series is thousands of years old. This field of research has witnessed two breakthroughs over the second-half of the twentieth century in the development of the Box-Jenkins methodology (\cite{box2015}) and the derivation of the Kalman filter (\cite{kalman1960}). While the ARIMA model underpinning the Box-Jenkins methodology has since become a standard in graduate curricula worldwide, the Kalman filter has found a great number of applications in addition to time series forecasting, including, but not limited to, object tracking, navigation and computer vision.

\subsection{Popular Approaches}
The Box-Jenkins methodology for analysing time series consists of three steps. Firstly, trends and seasonalities are removed from the training dataset by differentiating the time series iteratively\footnote{That is the following time series are constructed from the original time series $y_k$: $\Delta^{(1)} y_k := y_k - y_{k-1}$, \dots,  $\Delta^{(i+1)} y_k := \Delta^{(i)}y_k - \Delta^{(i)}y_{k-1}$ etc...} until an ergodic-stationary time series is found. Secondly, the ergodic-stationary time series is assumed to be a realisation of an ARMA process whose parameters are estimated, typically by maximizing the likelihood. Finally, the calibrated model is checked by analysing the sample autocorrelation of residuals.

Both the ARIMA model and the Kalman filter can be represented as discrete time \textit{linear Gaussian state space models} (see \cite[][chap. 3]{durbin12} or \cite[][chap. 10]{comkoop}). Discrete time state space models aim at inferring the state $x_k \in \mathbb{R}^p$ of a latent dynamic system, that is assumed to evolve according to Markovian dynamics, from some noisy observations $y_k  \in \mathbb{R}^q$. The models are characterised by a probability distribution for the initial latent state of the system $p(x_0)$, a Markovian transition dynamics $p(x_{k} \vert x_{k-1})$, and a measurement model $p(y_k \vert x_k)$. It is also assumed that the observations are independent from each other and from other latent states conditional on their associated latent states (i.e. $ \forall i \neq j ~y_i \perp y_j |x_i$ and $y_i \perp x_j |x_i$). The forecasting problem then consists of recursively determining $p(x_{k+n} \vert y_1, \dots, y_k)$ for some $n>0$.

In linear Gaussian state space models (LGSSM), the three distributions above are taken to be Gaussian, and the transition and measurement distributions have covariance matrices that do not depend on the state process. Moreover, it is assumed that 
\begin{align}
\text{E}\left( x_k \vert x_{k-1}\right)=F_kx_{k-1} + B_k u_k,  E\left(y_k \vert x_k \right) = H_k x_k, \nonumber
\end{align}
where $u_k$, $B_k$, $F_k$ and $H_k$ are given (or estimated off-line). Under this class of models the forecasting problem can be solved exactly and very efficiently using the Kalman filter equations. 
%The feasibility of the forecasting problem under this approach stems from the fact that $(x_k)_{k \in \mathbb{N}}$ and  $(y_k)_{k \in \mathbb{N}}$ are both discrete time Gaussian processes. 

LGSSMs can be extended in several ways. Alternatives have been proposed, including the popular \textit{extended Kalman filter}, that relax the linearity assumption in LGSSMs by postulating that 
\begin{align}
\text{E}\left( x_k \vert x_{k-1}\right)=f(x_{k-1}, u_k) \text{, } E\left(y_k \vert x_k \right) = h(x_k),\nonumber
\end{align}
for some (possibly nonlinear) functions $f$ and $h$, while retaining the Gaussian assumptions on the initial, transition and measurement distributions. Other approaches extend LGSSMs by relaxing the foregoing Gaussian assumptions. For instance, discrete measurement probability distributions such as the Poisson distribution or the negative binomial distribution would allow for count (as opposed to real-valued) observations (\cite{west85}), while leptokurtic probability distributions such as the student-t distribution or the Laplace distribution, used as measurement distribution, might improve the robustness of forecasts to outliers.  For a more comprehensive review of Box-Jenkins models and state space methods for time series analysis, we refer the reader to \cite{comkoop}, \cite{durbin12} and \cite{ham94}.

More generally, time series forecasting can be regarded as a regression problem, where one is interested in predicting future values and associated confidence bounds from historical observations. Gaussian process regression or GPR \cite[][chap. 2]{rasswill} provides a flexible Bayesian nonparametric framework for that purpose. The GPR approach to time series modelling consists of regarding the time series as arising from the sum of a continuous time Gaussian process $(z_t)_{t \geq 0}$ with mean function $m$ and covariance function $k_{\theta}$, and an independent Gaussian white noise: 
\begin{align}
&(z_t)_{t \geq 0} \sim \mathcal{GP}\left(m(.), k_{\theta}(.,.)\right),  ~  (\epsilon_t)_{t \geq 0} \sim \mathcal{GWN}(\sigma^2), \nonumber \\
&\forall t_i, ~ y_{t_i} = z_{t_i} + \epsilon_{t_i}\nonumber.
\end{align}
As a result, $(y_t)_{t \geq 0}$ is also a Gaussian process (GP), and the noise variance $\sigma^2$ and all other hyper-parameters can easily be inferred from some historical data $\{y_{t_0}, \dots, y_{t_k}\}$ by maximizing the corresponding multivariate Gaussian likelihood. Predictive distributions may then be obtained in closed-form. 

More recently, online \textit{passive-aggressive} algorithms have been proposed by \cite{pajmlr}, that allow for online linear and basis function regression with strong worst-case loss bounds.

\subsection{Limitations of Popular Approaches}
The critical assumption underpinning the Box-Jenkins methodology is that any time series, after enough iterative differentiations, will become \textit{ergodic-stationarity}, or more precisely will pass standard ergodic-stationarity statistical tests. The fundamental problem with this approach is that in practice, one will have access to finite samples that might not be sufficiently large or informative for the time series to pass any ergodic-stationarity test. It is in this spirit that \cite[][\S 3.10.1]{durbin12} noted that `in the economic and social fields, real series are never stationary however much differencing is done'. The ergodic-stationarity assumption is not a major limitation per se, but relying on the assumption that sufficient data have been collected to fully characterise the latent process (including testing for ergodic-stationarity) is a major limitation of the Box-Jenkins methodology. In that regards, \cite{kpss92} noted that `most economic time series are not very informative about whether or not there is a unit root'. The state space approach on the other hand does not require that the sample be informative enough to test for stationarity, which is why \cite{comkoop} and \cite{durbin12} argued it should be preferred to the Box-Jenkins methodology.
 
Additionally, the scalability of state space models is very appealing. However, the dynamics of the latent time series, characterised by $H_k$ ($h$ in the nonlinear case) is usually assumed to be known. Although this assumption is fairly mild in many engineering applications where the dynamics are derived from the laws of physics, when analysing time series arising from complex systems where the dynamics are not known, $H_k$ (resp. $h$) fully characterize the covariance function of the latent time series $(x_t)_{t \geq 0}$. Thus, hand-picking $H_k$ or $h$ would be problematic as the covariance function of the latent process should be flexibly learned from the data to appropriately uncover and exploit patterns. Unfortunately, no state space dynamics has been proposed in the literature, to the best of our knowledge, that guarantees that the implied covariance functions of the latent processes $(x_t)_{t \geq 0}$ can approximate arbitrarily well any stationary covariance function.
 
Families of covariance functions that are dense in the class of all continuous bounded covariance functions have recently been proposed by \cite{samo_gsk}, and can readily be used within the GPR framework. Unfortunately, with a cubic time complexity and a squared memory requirement, GPR scales poorly with the number of training samples. Sparse approximations have been proposed that scale linearly with the number of observations (\cite{quincand, sparsespectrum}). However, those methods are not good enough for streaming applications as they require loss of information through windowing to be of practical interest.

Passive-aggressive (PA) algorithms (\cite{pajmlr}) have constant time complexity and memory requirement for online linear regression problems. However, their kernelized versions have memory requirement and computational time that depend on the number of `support vectors', and might increase unboundedly with the number of observations (\cite{wang10}). Thus, PA algorithms in their original form are not appropriate for flexibly forecasting time series in a high throughput setting. \cite{wang10} have introduced PA algorithms with an additional memory or CPU `budget' constraint, but the authors restricted themselves to classification problems. Moreover, unlike probabilistic approaches, PA algorithms do not provide uncertainty around predictions. A Bayesian PA-like online learning approach has been proposed by \cite{shi14}. However, like all other kernel-PA algorithms, the kernel is assumed to be given and online learning of its parameters is not addressed by the authors.
 
%\subsection{Our Contributions}
\textbf{Our contibutions:} In this paper we build on the work of \cite{samo_gsk} to propose a state space model for analysing and forecasting time series that we prove is equivalent to GPR, under a family of kernels for the latent GP that is dense in the family of all stationary kernels, and that allows decoupling the flexibility of the kernel from the degree of smoothness of the latent GP. More importantly, we derive a novel scheme for online learning of hyper-parameters as a solution to a convex optimization problem, which happens to be a \textit{passive-aggressive} algorithm (\cite{pajmlr}) with \textit{stochastic gradient descent} updates (\cite{bottou98}). Overall, our model has constant time complexity and constant memory requirement, both for online learning of the hyper-parameters and for online prediction. 

The rest of the paper is structured as follows. In section \ref{sct:bk} we provide the intuition behind our approach and the related mathematical background. In section \ref{sct:mdl} we formally introduce our model. We illustrate that our approach outperforms competing approaches with some experiments in section \ref{sct:exp}. Finally, we discuss possible extensions and conclude in section \ref{sct:con}.

\section{INTUITION AND BACKGROUND}
\label{sct:bk}
In this paper we consider modelling real-valued time series. Our working assumption is that each sample $\{y_{t_0}, \dots, y_{t_N} \}$ of a time series arises as noisy observations of a `smooth'\footnote{By smooth we mean that the process is at least mean square continuous, and higher order derivatives might exist.} latent stochastic process $(z_t)_{t \geq 0}$ pertubated by an independent white noise $(\epsilon_t)_{t \geq 0}$:
\begin{align}
&\forall t, ~y_t = z_t + \epsilon_t, \label{eq:cond1}\\
&(\epsilon_t)_{t \geq 0} \sim \mathcal{WN}(\sigma^2), ~ (\epsilon_t)_{t \geq 0} \perp (z_t)_{t \geq 0}.\label{eq:cond2}
\end{align}
In financial econometrics applications for instance, the process $(z_t)_{t \geq 0}$ may represent the logarithm of the equilibrium or fair price of an asset, while $(\epsilon_t)_{t \geq 0}$ accounts for short term shocks, microstructure noise and so-called `fat-fingers'. More generally, $(z_t)_{t \geq 0}$ denotes the value of the time series of interest, after taking out hazards such as side effects, measurement noise, recording errors and so forth, accounted for by  $(\epsilon_t)_{t \geq 0}$.

Moreover, we will assume that the latent process $(z_t)_{t \geq 0}$ is a trend-stationary Gaussian process, i.e. a Gaussian process with a translation-invariant covariance function and an arbitrary mean function:
\begin{align}
(z_t)_{t \geq 0} \sim \mathcal{GP}(m(.), k(.)).\label{eq:cond3}
\end{align}
It is our working assumption that the class of trend-stationary Gaussian processes is large enough for our purpose. We refer the reader to \ref{app:hyp_disc} for a discussion on why this is a relatively mild assumption. From a Bayesian perspective, we are assuming \textit{a priori} that the process evolves about a deterministic trend, and that our \textit{prior} uncertainty on how the process might deviate from the trend is invariant by translation. In applications, the trend $m$ will typically be taken to lie in a parametric family of functions.

Furthermore, we assume that the white noise process $(\epsilon_t)_{t \geq 0}$ is Gaussian:
\begin{align}
(\epsilon_t)_{t \geq 0} \sim \mathcal{GWN}(\sigma^2).\label{eq:cond4}
\end{align}
We will discuss possible extensions to more robust noise models in section \ref{sct:con}. Our setup thus far is the same as GPR with a translation-invariant covariance function. Our next objective is to investigate whether we could exploit the scalability of the state space approach by providing an `equivalent' state space representation. More precisely, by `equivalent' we mean one in which the observable is the noisy time series value $y_t \in \mathbb{R}$, and for every collection of observation times $\{ t_0, \dots, t_N\}$, the joint distribution over $\left(y_{t_0}, \dots, y_{t_k}\right)$ as per the state space representation is the same as that implied by Equations (\ref{eq:cond1}) to (\ref{eq:cond4}). Thus, we need to construct an appropriate Markov process $(x_t)_{t \geq 0}, x_t \in \mathbb{R}^p$ that is related to $(z_t)_{t \geq 0}$ and from which the state space dynamics will be deduced. We will then derive the measurement dynamics from Equation  (\ref{eq:cond1}). As $(x_t)_{t \geq 0}$ should be a Markov process, intuitively $x_t$ should contain all information about present and past values of the latent time series $(z_t)_{t \geq 0}$ if we want the state space representation to be equivalent to Equations (\ref{eq:cond1}) to (\ref{eq:cond4}). Thus, it seems natural to consider that $(z_t)_{t \geq 0}$ is a coordinate process of the vector state process  $(x_t)_{t \geq 0}$.

\underline{Case 1}: $\forall t, ~ x_t = z_t$

If we take the state process to be the latent time series itself, $(z_t)_{t \geq 0}$ needs to be Markovian. However, this would imply that the latent time series will either not be mean square differentiable, or, as the following lemma states, it will almost surely be equal to the trend plus a constant, which might be too restrictive.
\begin{lemma}If  a real-valued trend-stationary Gaussian process with differentiable mean function is mean square differentiable and Markovian, then it has a constant covariance function (or equivalently it is almost surely equal to its mean function plus a constant).
\end{lemma}
\begin{proof}
See \ref{sct:only_smooth_mgp_}.
\end{proof}
Intuitively, the memorylessness of the Markov assumption conflicts with the memoryfulness of the differentiability assumption, so that the choice $x_t = z_t$ might only be appropriate when assuming that the smooth latent time series $(z_t)_{t \geq 0}$ is mean square continuous but not differentiable. An example covariance function yielding a continuous trend-stationary Markov Gaussian process is the Mat\'{e}rn-$\frac{1}{2}$ kernel $k_{\text{exp}}$ \cite[see][Appendix \S B.2]{rasswill}. The equivalent state space representation is easily derived from standard Gaussian identities as:
\begin{equation*}
\begin{cases}
z_{t_0} \sim \mathcal{N}\left(0, k_{\text{exp}}(0)\right),\\
z_{t_k} = m(t_k) + \alpha_k z_{t_{k-1}} +  \sqrt{k_{\text{exp}}(0)\left(1- \alpha_k^2\right)}\xi_{t_k}\\
y_{t_k} = z_{t_k} + \epsilon_{t_k}
\end{cases}
\end{equation*}
where $\alpha_k = \frac{k_{\text{exp}}(t_k, t_{k-1})}{k_{\text{exp}}(0)}$, and $(\xi_t)_{t \geq 0} \sim \mathcal{GWN}(1)$.

\underline{Case 2}: $\forall t, ~ x_t = \left(z_t, z_t^{(1)}, \dots, z_t^{(p)}\right)$

In many applications, we might be interested in postulating that the latent time series $(z_t)_{t \geq 0}$ is smoother, for instance $p$ times mean square differentiable. In this case, we derive our intuition for what might be good candidates to complete the state process into an appropriate Markov vector process from the theory of analytic functions underpinning implementations of scientific functions in most programming languages. Conceptually, the aforementioned theory implies that under mild conditions, the value of an infinitely smooth function and its derivatives $f(t_k), f^{(1)}(t_k), \dots, f^{(i)}(t_k), \dots$ at a given point $t_k$ are sufficient to fully characterise the entire function elsewhere on the domain. Although we only require $(z_t)_{t \geq 0}$ to be $p$ times differentiable, it is our hope that by augmenting the state process with the $p$ consecutive derivatives of $(z_t)_{t \geq 0}$, we could ensure that (i) each state variable $x_t$ contains all information about past values of the latent time series $(z_u)_{u \leq t}$ or equivalently the state process $(x_t)_{t \geq 0}$ is Markovian, (ii) the latent time series $(z_t)_{t \geq 0}$ is as smooth as desired, and (iii) the covariance function of the latent time series is not strongly restricted.

When $(z_t)_{t \geq 0}$ is a Gaussian process that is $p$ times continuously differentiable in the mean square sense, we denote as \textit{p-derivative Gaussian process} ($p$DGP) the vector Gaussian process $(x_t)_{t \geq 0}$ where $x_t=\left(z_t, z^{(1)}_t, \dots, z^{(p)}_t\right)$ and $(z^{(i)})_{t \geq 0}$ is the $i$-th order mean square derivative of  $(z_t)_{t \geq 0}$. We say that $(z_t)_{t \geq 0}$ is a $p$-Markov Gaussian process ($p$M-GP) when it is $p$ times continuously differentiable and the corresponding $p$DGP is Markovian.\footnote{A $0$M-GP ($p=0$) is a Gaussian Markov process.} Stationary $p$M-GPs have been extensively studied in the seminal paper by \cite{doob44}.\footnote{In Doob's paper a $p$M-GP is denoted a t.h.G.M.$_{(p-1)}$ process.} In particular, the fundamental theorem below \cite[see][Theorem 4.9 ii]{doob44} provides necessary and sufficient conditions on our covariance function $k$ for the augmented state process to be Markovian when the trend is constant.
\begin{theorem} A real-valued stationary Gaussian process $(z_t)_{t \geq 0}$ with integrable covariance function $k$ is $p$-Markov if and only if $k$ admits a spectral density of the form
\begin{equation}
\label{eq:pm_cond}
\forall s> 0, ~S(s) := \int k(\tau) e^{-2\pi i s\tau}d\tau = \frac{c}{|A(is)|^2},
\end{equation}
where $c>0$ and $A$ is a polynomial of degree $p+1$ with real-valued coefficients and no purely imaginary root.
\end{theorem}
\begin{corollary}\label{coro:coro} A stationary Gaussian process with covariance function the Mat\'{e}rn kernel
\begin{equation*}
k_{\text{ma}}(\tau; k_0, l, \nu) = k_0 \frac{2^{1-\nu}}{\Gamma(\nu)} \left(\frac{\sqrt{2\nu} \vert \tau \vert}{l} \right)^\nu K_\nu \left(  \frac{\sqrt{2\nu} \vert \tau \vert}{l}\right), 
\end{equation*}
where $k_0, l >0$, $\Gamma$ is the Gamma function, $K_\nu$ is the modified Bessel function of second kind, and with $\nu = p + \frac{1}{2}, ~p \in \mathbb{N}$ is $p$-Markov.
\end{corollary}
\begin{proof}
The spectral density of a Mat\'{e}rn-$(p+\frac{1}{2})$ covariance function is of the form 
\[\frac{c}{(\alpha + 4\pi^2 s^2)^{p+1}}= \frac{c}{\vert (\sqrt{\alpha} + 2\pi i s)^{p+1}\vert^2}\] 
with $c, \alpha>0$ \cite[see][]{rasswill}.
\end{proof}
Before deriving the state space representation of trend-stationary GPs with Mat\'{e}rn-$(p+\frac{1}{2})$ covariance functions, we recall that a $p$DGP is a vector Gaussian process and that in the trend-stationary case
\begin{equation*}
\forall ~0 \leq i, j \leq p, ~\text{cov}\left(z_u^{(i)}, z_v^{(j)} \right) = (-1)^j k^{(i+j)}(u-v),
\end{equation*}
where we use the superscript $^{(i)}$ to denote $i$-th order derivative or the orginal function (or process) when $i=0$ \cite[see][\S 2.6]{adlertaylor}. If we further denote 
\[K_{u,v} := \left\{\text{cov}\left(z_u^{(i)}, z_v^{(j)} \right)\right\}_{0 \leq i,j \leq p}\]
the covariance matrix between $x_u$ and $x_v$, as
\[K_{u\vert v} := K_{u, u} - K_{u, v} K_{v, v}^{-1}K_{v, u}\]
the auto-covariance matrix of $x_u$ conditional on $x_v$, and as $L_{u\vert v}$ any squared matrix satisfying\footnote{The Cholesky factorisation and the SVD provide such a decomposition.}
\[K_{u\vert v} = L_{u\vert v}L_{u\vert v}^T,\]
then we obtain the following equivalent LGSSM representation (see \ref{sct:pmat_lgssm} for details):
\begin{equation*}
\begin{cases}
x_{t_0} \sim \mathcal{N}(0, K_{t_0, t_0}) \\
x_{t_k} = F_{t_k} x_{t_{k-1}} + L_{t_k \vert t_{k-1}}\xi_{t_k}\\
y_{t_k} = m(t_k) + H^Tx_{t_k} + \epsilon_{t_k}
\end{cases}
\end{equation*}
with $F_{t_k} = K_{t_k, t_{k-1}} K_{t_{k-1}, t_{k-1}}^{-1}, ~H=(1, 0, \dots, 0),$ and where $(\xi_t)_{t \geq 0}$ is now a $(p+1)$-dimensional standard Gaussian white noise. 

The idea of using derivative information to construct a scalable alternative to GPs for MCMC inference has recently been developed in \cite{samo_sgp}. The idea of representing GPR as a state space model with a Markovian dynamics involving derivatives has been explored by \cite{sarkka13},  where the authors proposed approximating the spectral density of a given covariance function by a function in the form of Equation (\ref{eq:pm_cond}). Unfortunately, deriving such approximations might be tedious in certain cases. More importantly, postulating that the covariance function has spectral density of the form Equation (\ref{eq:pm_cond}) is still not flexible enough as it creates a link between the degree of differentiability of the latent process and the number of modes of the spectral density. To see why, we note that critical points (or local extrema) of $\frac{c}{|A(is)|^2}$ and $|A(is)|^2$ are the same, and that $|A(is)|^2$ being a polynomial of degree $2p+2$ in $s$, it can have at most $2p +1$ critical points. This means that if for instance we would like the latent function to be at most once differentiable ($p=1$), the approach of \cite{sarkka13} would restrict the spectral density to have at most $3$ modes, rather than flexibly learning the number of modes from the data. This also holds when the spectral density is approximated by a rational function in the extension suggested by \cite{sarkka13}, as the degree of the numerator needs to be smaller than the degree of the denominator, and the later is linked to the degree of differentiability of the latent GP.

Families of kernels, namely \textit{spectral Mat\'{e}rn kernels} (\cite{samo_gsk}), do exist that are dense in the family of all stationary covariance functions, and allow postulating or learning the degree of differentiability of the corresponding GP. In the next section, we build on the state space representation of $p$-Markov Mat\'{e}rn Gaussian processes to provide an exact state space representation of our time series model (Equations (\ref{eq:cond1}) to (\ref{eq:cond4}))  when the latent time series $(z_t)_{t \geq 0}$ is assumed to have a \textit{spectral Mat\'{e}rn-($p+\frac{1}{2}$)} covariance function (\cite{samo_gsk}). Moreover, we derive a novel \textit{passive-aggressive} algorithm for online learning of the hyper-parameters of the state space model, and that has constant memory requirement and constant time complexity.

\section{OUR MODEL}
\label{sct:mdl}
We start by introducing a state space time series model we refer to as the $p$-Markov Gaussian process filter ($p$M-GP filter) that we prove is equivalent to Gaussian process regression under a \textit{spectral Mat\'{e}rn} kernel.
\subsection{The $p$-Markov Gaussian Process Filter}
\begin{definition}\label{def:pm-gpf}($p$M-GP filter) We denote $p$-Markov Gaussian process filter ($p$M-GP filter) a state space model of the form:
\begin{equation*}
\begin{cases}
\perp \left\{{}^{i}_{c}x_{t_0}, {}^{i}_{s}x_{t_0}\right\}_{i=0}^n, \perp \left\{({}^i_{c}\xi_{t})_{t \geq 0}, ({}^i_{s}\xi_{t})_{t \geq 0}\right\}_{i=0}^n \\
\forall i, ~ {}^i_{c}x_{t_{k-1}} \perp {}^i_{c}\xi_{t_k},~ {}^i_{s}x_{t_{k-1}} \perp {}^i_{s}\xi_{t_k} \\
\forall i, ~ {}^i_{c}x_{t_{k}} \perp \epsilon_{t_k},~ {}^i_{s}x_{t_{k}} \perp \epsilon_{t_k} \\
\forall i, ~{}^i_{c}x_{t_0}, {}^i_{s}x_{t_0} \sim \mathcal{N}(0, {}^iK_{t_0, t_0})\\
\forall i, ~{}^i_{c}x_{t_k} = {}^iF_{t_k} {}^i_{c}x_{t_{k-1}}  + {}^iL_{t_k \vert t_{k-1}}{}^i_{c}\xi_{t_k}\\
\forall i, ~{}^i_{s}x_{t_k} = {}^iF_{t_k} {}^i_{s}x_{t_{k-1}}  + {}^iL_{t_k \vert t_{k-1}}{}^i_{s}\xi_{t_k}\\
y_{t_k} = H^T\sum_{i=0}^{n} \left(\cos(\omega_i t_k){}^i_{c}x_{t_k} + \sin(\omega_i t_k){}^i_{s}x_{t_k} \right)\\
~~~~~ + m(t_k) +  \epsilon_{t_k}
\end{cases}
\end{equation*}
where $\perp$ denotes mutual independence,  $({}^i_{c}\xi_{t})_{t \geq 0}$ and $({}^i_{s}\xi_{t})_{t \geq 0}$ are $(p+1)$-dimensional standard Gaussian white noises, $(\epsilon_t)_{t \geq 0}$ is a scalar Gaussian white noise with variance $\sigma^2$, $^{i}K_{u, v}$ is the cross-covariance matrix of a $p$DGP with Mat\'{e}rn kernel $k_{\text{ma}}(\tau; k_{0i}, l_i, p+\frac{1}{2})$, and where $H, {}^iF_{t_k}, {}^iL_{t_k \vert t_{k-1}}$ are as in the previous section (replacing $K_{u, v}$ by $^{i}K_{u,v}$).
\end{definition}
The following proposition establishes that the $p$M-GP filter is equivalent to GPR under a spectral Mat\'{e}rn kernel,  and a possibly non-constant mean function.
\begin{proposition}\label{prop:eqv}Let $(\hat{z}_t)_{t \geq 0}$ be a trend-stationary Gaussian process with mean function $m$ and \textit{spectral Mat\'{e}rn} covariance function
\begin{equation}
k_{\text{sma}}(\tau) = \sum_{i=0}^{n} k_{\text{ma}}\left(\tau; k_{0i}, l_i, p + \frac{1}{2}\right)\cos(\omega_i \tau).
\end{equation}
Let $(\hat{\epsilon}_t)_{t \geq 0}$ be a Gaussian white noise with variance $\sigma^2$ that is independent from $(\hat{z}_t)_{t \geq 0}$, and $(\hat{y}_t)_{t \geq 0}$ the process defined as \[\forall t \geq 0, ~\hat{y}_t = \hat{z}_t + \hat{\epsilon}_t.\]
Finally, let $(y_t)_{t \geq 0}$ be the observation process of the $p$M-GP filter with the same parameters $m, n, p, \sigma, \{k_{0i}, l_i, \omega_i\}_{i=0}^n$. Then, the processes $(y_t)_{t \geq 0}$ and $(\hat{y}_t)_{t \geq 0}$ have the same law, or equivalently:
\[\forall t_0 < \dots < t_N, ~ (y_{t_0}, \dots, y_{t_N}) \sim (\hat{y}_{t_0}, \dots, \hat{y}_{t_N}).\]
\end{proposition}
\begin{proof}
See \ref{sct:pm_gp_fil}.
\end{proof}
The above proposition has profound implications. Firstly, the $p$M-GP filter is the first equivalent state space representation of the Gaussian process regression model (on a unidimensional input space) under a family of covariance functions that is dense in the family of all stationary covariance functions. This representation allows combining the great flexibility of GPR, with the unmatched scalability of state space models, without resorting to approximations. Secondly, the model allows controlling the degree of differentiability of the latent time series independently from the flexibility of the covariance function of the latent time series. Finally, unlike models such as ARIMA that assume equally spaced observations, as any GPR model, the $p$M-GP filter is inherently asynchronous and naturally copes with missing data.
\subsection{Solution to the Forecasting Problem}
We note that the $p$M-GP filter can be rewritten as 
\begin{equation}
\label{eq:solu}
\begin{cases}
\bm{x}_{t_0} \sim  \mathcal{N}(0, \bm{K}_{t_0, t_0})\\
\bm{x}_{t_k} = \bm{F}_{t_k}\bm{x}_{t_{k-1}} + \bm{L}_{t_k \vert t_{k-1}} \bm{\xi}_{t_k}\\
y_{t_k} = m(t_k) + \bm{H}_{t_k}^T \bm{x}_{t_k} +  \epsilon_{t_k}.
\end{cases}
\end{equation}
Here, $\bm{x}_{t_k}$ (resp. $\bm{\xi}_{t_k}$) is the $2(p+1)(n+1)$ vector obtained by stacking up the the vectors $\{ \dots, {}^i_{c}x_{t_k}, {}^i_{s}x_{t_k}, \dots\}$ (resp. $\{ \dots, {}^i_{c}\xi_{t_k}, {}^i_{s}\xi_{t_k}, \dots\}$). Similarly, $\bm{K}_{t_0, t_0}$ (resp. $\bm{F}_{t_k}, \bm{L}_{t_k \vert t_{k-1}}, \bm{K}_{t_k \vert t_{k-1}}$) is the $2(p+1)(n+1) \times 2(p+1)(n+1)$ block diagonal matrix whose $2i$-th and $(2i + 1)$-th diagonal blocks are both ${}^iK_{t_0, t_0}$ (resp. ${}^iF_{t_k}, {}^iL_{t_k \vert t_{k-1}}, {}^iK_{t_k \vert t_{k-1}}$). Finally, $\bm{H}_{t_k}$ is the $2(p+1)(n+1)$ vector with coordinates $0$ except at indices multiple of $p+1$, and $\forall i \in [0 \dots n], ~\bm{H}_{t_k}[2i(p+1)] = \cos(\omega_i t_k)$ and $\bm{H}_{t_k}[(2i+1)(p+1)] = \sin(\omega_i t_k)$.

Thus the $p$M-GP filter is a LGSSM, and the forecasting problem can be solved exactly, iteratively and in closed-form, and with memory requirement and time complexity both constant in the total number of observations. Using standard Kalman filter and Gaussian processes techniques (see \ref{sct:solu_fill}) we obtain:
\begin{equation}
\label{eq:solu2}
\begin{cases}
\forall t> t_{k-1}, ~\bm{x}_t \vert y_{t_0 : t_{k-1}}  &\sim  \mathcal{N}(\bm{m}_t^{-}, \bm{P}_t^{-})\\
\forall t> t_{k-1}, ~y_t \vert y_{t_0 : t_{k-1}}  &\sim  \mathcal{N}(m(t) + \bm{H}_t^T\bm{m}_t^{-}, v_t^{-})\\
\forall t> t_{k-1}, ~z_t \vert y_{t_0 : t_{k-1}}  &\sim  \mathcal{N}(m(t) + \bm{H}_t^T\bm{m}_t^{-},v_t^{-}-\sigma^2 )\\
\forall t \geq  t_{k-1}, ~\bm{x}_t \vert  y_{t_0 : t} &\sim \mathcal{N}(\bm{m}_t, \bm{P}_t)\\
\forall t \geq  t_{k-1}, ~z_t \vert  y_{t_0 : t} &\sim \mathcal{N}(m(t) + \bm{H}_t^T\bm{m}_t, v_t)
\end{cases}
\end{equation}
with $y_{t_0: t_{k-1}} = \{y_{t_0}, \dots, y_{t_{k-1}}\}$, $z_t = m(t) + \bm{H}_t^T \bm{x}_t$ and
\begin{align}
&\underline{\text{Prediction step:}}\nonumber\\
&\bm{m}_{t_0}^{-} = 0, ~ \bm{P}_{t_ 0}^{-} = \bm{K}_{t_0, t_0} \label{eq:pred} \\
&\forall k \geq 1, t> t_{k-1}
\begin{cases}
\bm{m}_t^{-} &= \bm{F}_t \bm{m}_{t_{k-1}}\\
\bm{P}_t^{-} &= \bm{F}_t \bm{P}_{t_{k-1}} \bm{F}_t^T + \bm{K}_{t \vert t_{k-1}}
\end{cases}\nonumber \\
&\underline{\text{Update step:}}\nonumber\\
&\forall t
\begin{cases}
v_t^{-} &= \bm{H}_t^T \bm{P}_{t}^{-} \bm{H}_t + \sigma^2 \\
\bm{e}_t^{-} &= y_t - m(t) - \bm{H}_t^T\bm{m}_t^{-} \\
\bm{G}_t &= \frac{1}{v_t^{-}}\bm{P}_t^{-}\bm{H}_t \\
\bm{m}_t &= \bm{m}_t^{-} + \bm{e}_t^{-}\bm{G}_t \\
\bm{P}_t &= \bm{P}_t^{-} - v_t^{-}\bm{G}_t\bm{G}_t^T\\
v_t &= \bm{H}_t^T\bm{P}_t\bm{H}_t
\end{cases}.\label{eq:update}
\end{align}
\subsection{Online Learning of Hyper-Parameters}
Noting that $y_{t_0} \sim \mathcal{N}(m(t_0), v_{t_0}^{-})$, and that 
\begin{equation*}
\log p(y_{t_0:t_N}) = \log p(y_{t_0}) + \sum_{k=1}^T \log p(y_{t_k} \vert y_{t_0:t_{k-1}}),
\end{equation*}
it follows that maximum likelihood inference of the hyper-parameters of $m$, $\sigma$, and $\{k_{0i}, l_i, \omega_i\}_{i=0}^n$ may be achieved with linear time complexity $\mathcal{O}(T)$ and with constant memory requirement using Equations (\ref{eq:solu2}, \ref{eq:pred}, \ref{eq:update}). However, to be of practical interest in streaming applications, this time complexity requires loss of information through windowing, which might be detrimental to forecasting performance. We herein propose an alternative online approach that has constant time complexity.

So far we have assumed that the hyper-parameters are constant over time. However, if we extend the model to iteration-specific hyper-parameters, the solution to the forecasting problem (Equations (\ref{eq:solu2}, \ref{eq:pred}, \ref{eq:update})) will provably remain unchanged. The intuition behind our approach to online learning of the hyper-parameters is borrowed from maximum likelihood inference when the hyper-parameters are constant, and online passive-aggressive algorithms for regression, classification and semi-supervised learning (\cite{pajmlr,wang10,chang10}).  Let $\bm{\theta}_{t_k}$ be the vector of hyper-parameters (or log hyper-parameters for positive hyper-parameters) prevailing at time $t_k$. We recall that when hyper-parameters are constant, we may write
\begin{align}
\log p_{{\bm{\theta}}}\left(y_{t_0:t_T}\right) = \log p_{{\bm{\theta}}}\left(y_{t_0}\right) + \sum_{k=1}^{T}  \log p_{{\bm{\theta}}}\left(y_{t_k}\vert y_{t_0:t_{k-1}}\right).\nonumber
\end{align}
Let us define
\begin{equation*}
\mathcal{L}^{l}_{t_k}(\bm{\theta}) := 
\begin{cases}
\log p_{\bm{\theta}}\left(y_{t_0}\right) & \text{if } k=0 \\
\log p_{\bm{\theta}}\left(y_{t_k}\vert y_{t_0:t_{k-1}}\right) & \text{if } k \neq 0
\end{cases}.
\end{equation*}
At time $t_k$, $\mathcal{L}^{l}_{t_k}(\bm{\theta})$ represents the (log) likelihood that $\bm{\theta}$ explains the new observation $y_{t_k}$ given all previous observations (and hyper-parameters) and thus can be regarded as a \textit{local utility} function that we should aim to maximize. However, we should also be mindful of the `utility' of $\bm{\theta}$ in accounting for all previously observed data $(y_{t_0}, \dots, y_{t_{k-1}})$, which can be achieved by ensuring that the new update is not be too far away from $\bm{\theta}_{t_{k-1}}$. Deriving online learning updates as solutions to a constrained optimization problem that provides a trade-off between retaining information from previous iterations and locally reducing a loss function has been advocated for a long time in the online learning literature (\cite{littlestone89, kivinen97,helmbold99,pajmlr,wang10,chang10}). The approach we adopt is largely inspired by the PA-II algorithm of \cite{pajmlr}. We first consider the problem:
\begin{align}
\label{prob:1}
\begin{cases}
\underset{\bm{\theta}}{\text{min}} ~|| \bm{\theta} - \bm{\theta}_{t_{k-1}}||^2 + c_k \xi^2\\
\text{ s.t. } \max\left(-\epsilon -\mathcal{L}^{l}_{t_k}(\bm{\theta}), 0\right) \leq \xi
\end{cases},
\end{align}
where $\epsilon \geq 0$ and $c_k>0$. $c_k$ can be regarded as an aggressiveness parameter: the larger $c_k$ the more weight is given to increasing the local utility compared to retaining information from previous iterations. $\epsilon$ on the other hand is a margin or tolerance parameter that allows for some degree of local sup-optimality, thereby helping to avoid overfitting. Although this problem provides a suitable trade-off, it is tedious to solve analytically,  and it might not have a closed-form solution. However, noting that the objective function of problem (\ref{prob:1}) aims at minimizing $|| \bm{\theta} - \bm{\theta}_{t_{k-1}}||$, we may  replace $\mathcal{L}^{l}_{t_k}(\bm{\theta})$ with its first order Taylor expansion at $\bm{\theta}_{t_{k-1}}$. We adopt the resulting problem for online learning of the hyper-parameters:
\begin{align}
\label{prob:2}
\begin{cases}
\bm{\theta}_{t_k} = \underset{\bm{\theta}}{\text{argmin}} ~|| \bm{\theta} - \bm{\theta}_{t_{k-1}}||^2 + c_k \xi^2\\
\text{ s.t. } \max\left(-\epsilon -\hat{\mathcal{L}}^{l}_{t_k}(\bm{\theta}), 0\right) \leq \xi
\end{cases},
\end{align}
with $\hat{\mathcal{L}}^{l}_{t_k}(\bm{\theta}) = \mathcal{L}^{l}_{t_k}(\bm{\theta}_{t_{k-1}}) + \nabla \mathcal{L}^{l}_{t_k}(\bm{\theta}_{t_{k-1}})^T (\bm{\theta} - \bm{\theta}_{t_{k-1}})$. The above optimization problem is convex and has closed-form solution (see \ref{sct:optim_solu} for the derivation, and \ref{sct:optim_upd} for the derivation of $\nabla \mathcal{L}^{l}_{t_k}(\bm{\theta})$):
\begin{align}
\label{prob:solu}
\bm{\theta}_{t_k} = \bm{\theta}_{t_{k-1}} + c_k\frac{\max\left(-\epsilon -\mathcal{L}^{l}_{t_k}\left(\bm{\theta}_{t_{k-1}}\right), 0 \right)}{1 + c_k \| \nabla \mathcal{L}^{l}_{t_k}(\bm{\theta}_{t_{k-1}}) \|^2}\nabla \mathcal{L}^{l}_{t_k}(\bm{\theta}_{t_{k-1}}).
\end{align}
Noting that the parameters do not change when $\mathcal{L}^{l}_{t_k}(\bm{\theta}_{t_{k-1}}) > -\epsilon$, it follows that our learning scheme can be regarded as a \textit{passive-aggressive} algorithm with \textit{stochastic gradient descent} update (\cite{bottou98}). Moreover, as
\begin{align}
&\hat{\mathcal{L}}^{l}_{t_k}(\bm{\theta}_{t_k}) - \hat{\mathcal{L}}^{l}_{t_k}(\bm{\theta}_{t_{k-1}}) \nonumber \\
&= \max\left(-\epsilon-\mathcal{L}^{l}_{t_k}\left(\bm{\theta}_{t_{k-1}}\right), 0 \right) \frac{c_k \| \nabla \mathcal{L}^{l}_{t_k}(\bm{\theta}_{t_{k-1}}) \|^2}{1 + c_k \| \nabla \mathcal{L}^{l}_{t_k}(\bm{\theta}_{t_{k-1}}) \|^2} \geq 0, \nonumber
\end{align}
each change in the hyper-parameters increases the approximate conditional log likelihood $\hat{\mathcal{L}}^{l}_{t_k}$. Moreover, the one-side $\epsilon$-insensitive loss term $\max\left(-\epsilon-\mathcal{L}^{l}_{t_k}\left(\bm{\theta}_{t_{k-1}}\right), 0 \right)$ guards against overfitting when an update is required. Algorithm \ref{alg:fore_pmgp} summarizes online forecasting and hyper-parameters learning under the $p$M-GP filter.
\begin{algorithm}[ht]
   \caption{Forecasting with the $p$M-GP filter.}
   \label{alg:fore_pmgp}
\begin{algorithmic}
\STATE {\bfseries In:} $\{ y_{t_k}\}_{k \geq 0}$, average sampling frequency $F_s$.
\STATE {\bfseries Out:} Predictive probability density functions $\{ \dots, p(z_t \vert y_{t_0}:y_{t_k}), \dots \}, t \geq t_k$.
\\\hrulefill
\STATE  Set $\omega_i = \frac{1+i}{1+n}\pi F_s$, set the other parameters to $0$.
\FORALL{$(t_k, y_{t_k})$}
\STATE Evaluate $\bm{\theta}_{t_k}$ using Eq. (\ref{prob:solu}).
\STATE Run the prediction step Eqs. (\ref{eq:pred}) with $t=t_k$.
\STATE Run the update step Eqs. (\ref{eq:update}) with $t=t_k$.
\STATE For $t \geq t_k$ get $p(z_t \vert y_{t_0}:y_{t_k})$ from Eqs. (\ref{eq:solu2}, \ref{eq:pred}).
\ENDFOR
\end{algorithmic}
\end{algorithm}

In order to control the aggressiveness of our learning algorithm in a manner that is consistent across datasets, we rewrite the aggressiveness parameter in the form $c_k := c  \frac{\| \bm{\theta}_{t_{k-1}}\|^2}{\left(\epsilon +\hat{\mathcal{L}}^{l}_{t_k}(\bm{\theta}_{t_{k-1}})\right)^2}, c>0$, that arises by normalizing the term $|| \bm{\theta} - \bm{\theta}_{t_{k-1}}||^2$ (resp. $\xi^2$) in the objective function of problem (\ref{prob:2}) by $\| \bm{\theta}_{t_{k-1}}\|^2$ (resp. $\left(\epsilon +\hat{\mathcal{L}}^{l}_{t_k}(\bm{\theta}_{t_{k-1}})\right)^2$).

\section{EXPERIMENTS}
\label{sct:exp}
\begin{figure*}[t!]
\centering
\includegraphics[width=0.4\textwidth]{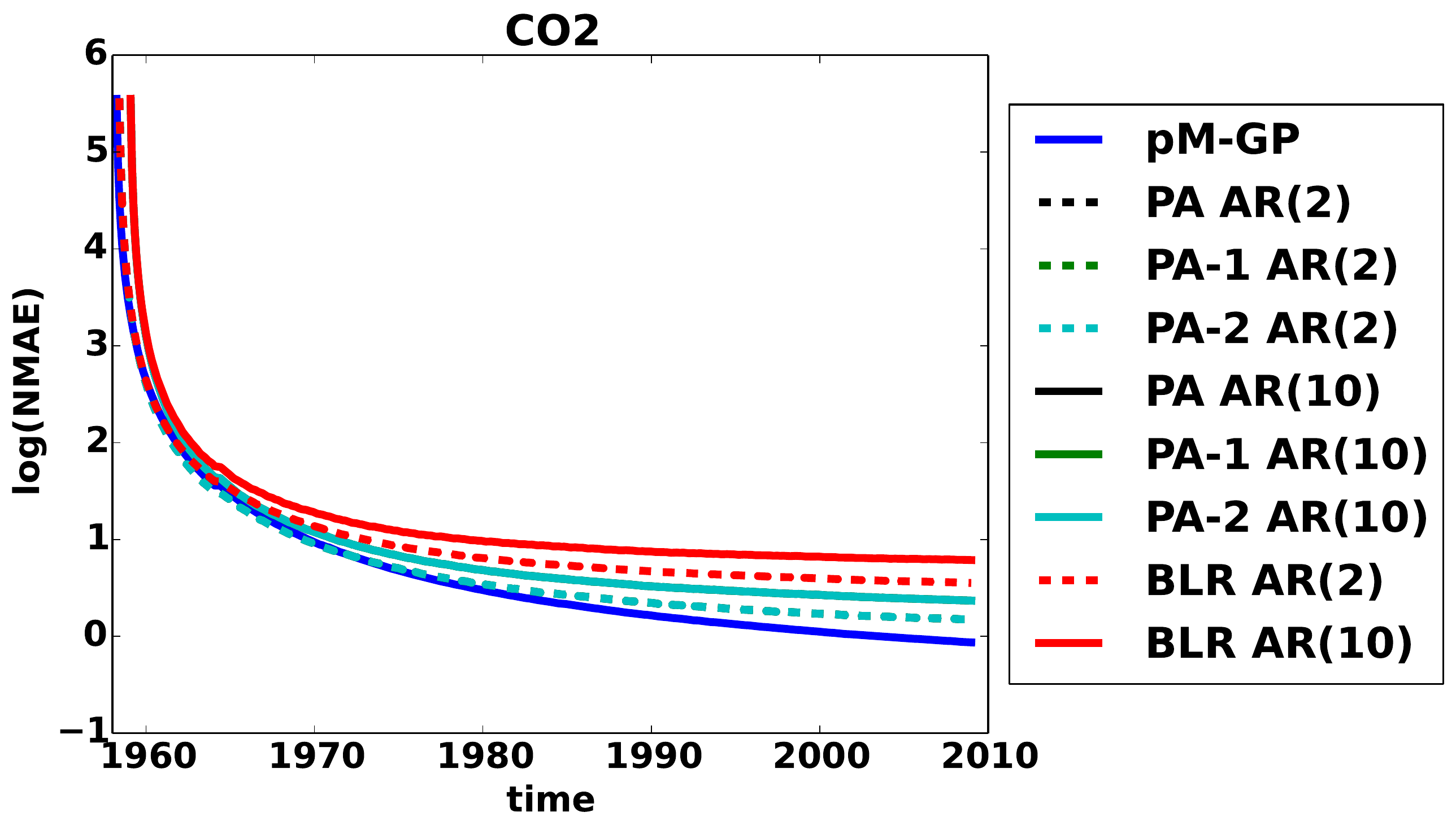}
\includegraphics[width=0.29\textwidth]{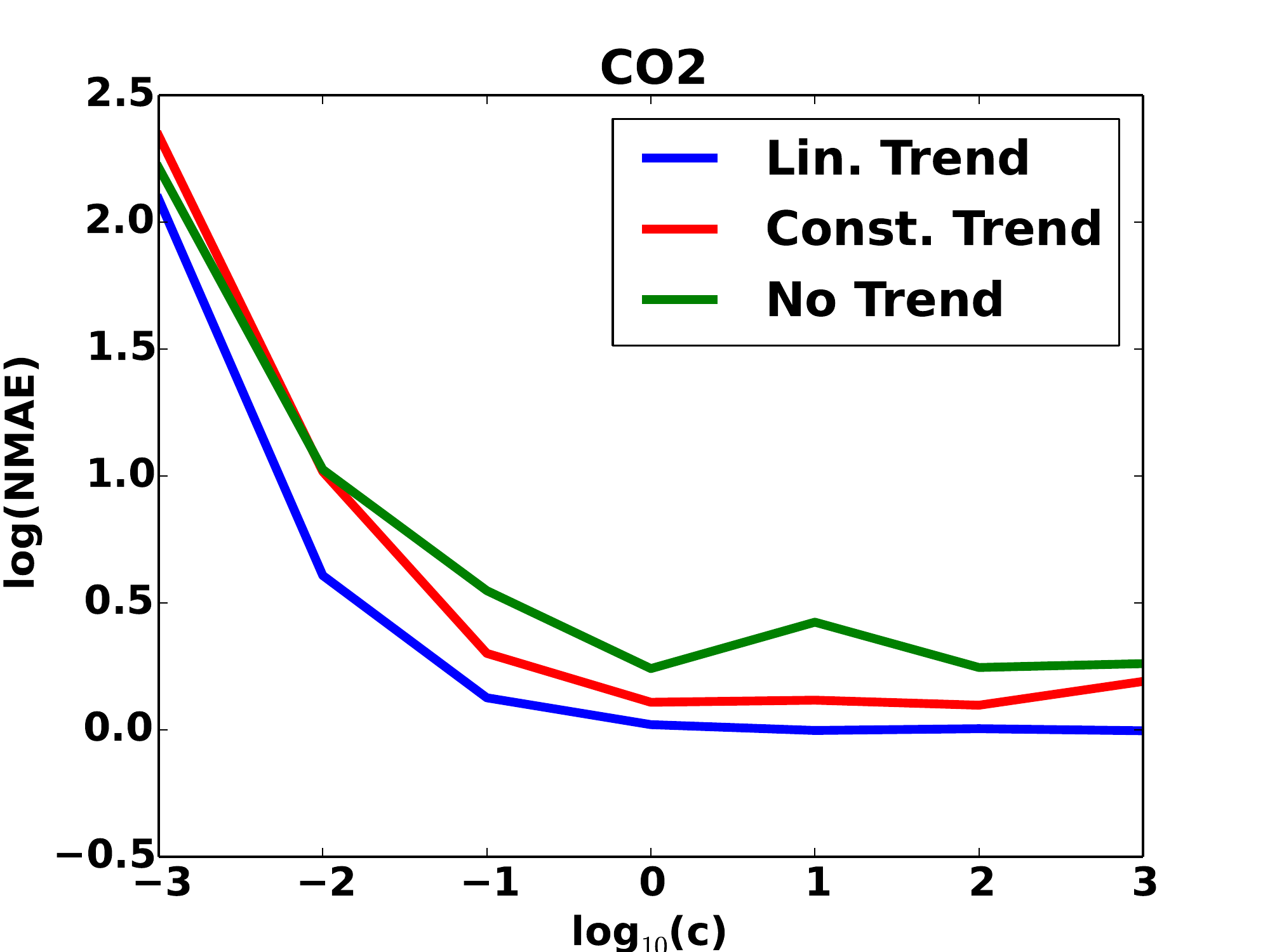}
\includegraphics[width=0.29\textwidth]{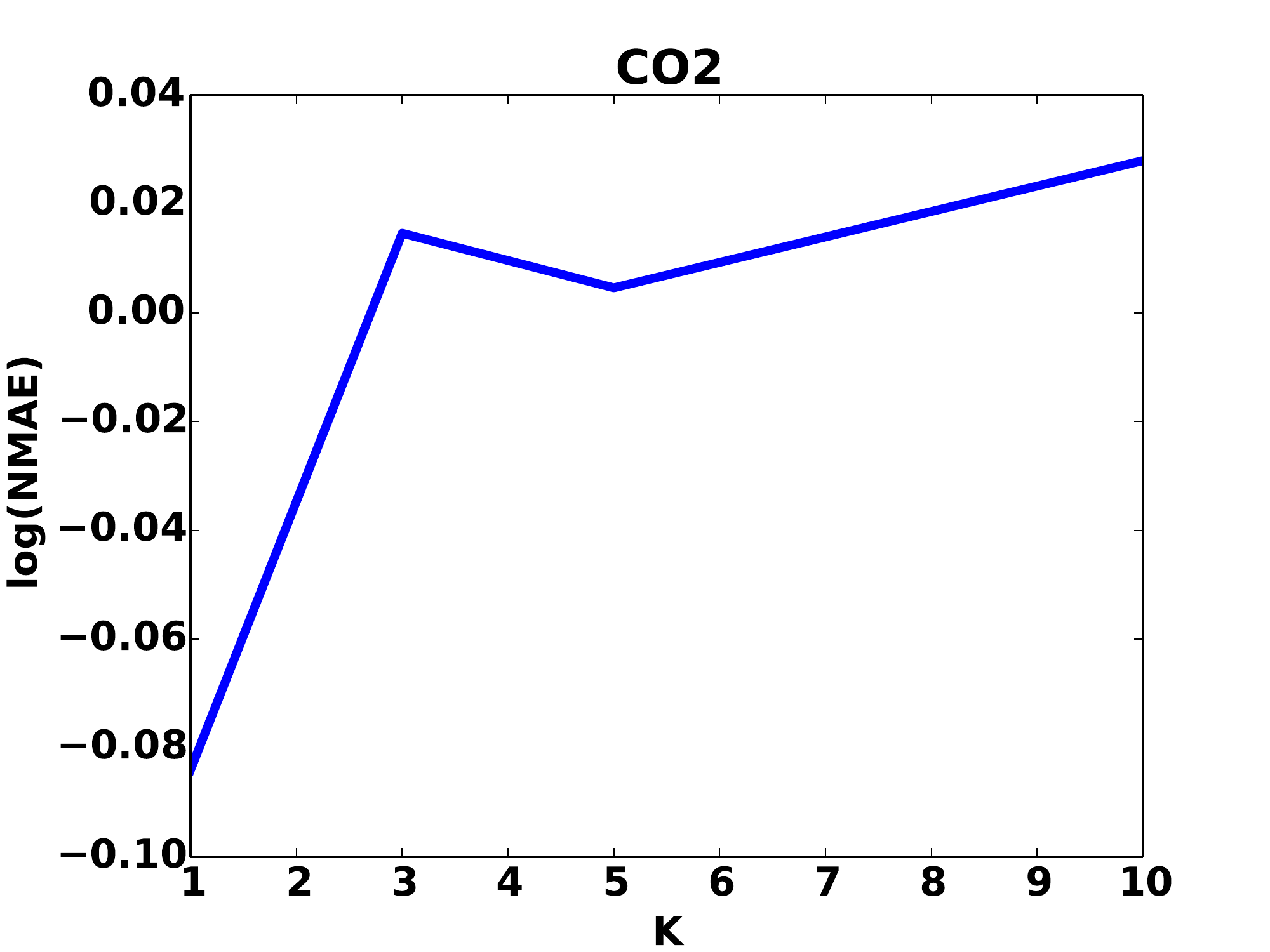}
\caption{CO2 experiment: (left) running normalized mean absolute error (NMAE), (middle) effect of the $p$M-GP aggressiveness parameter $c$ on NMAE, and (right) effect of the number of spectral components $K$ on NMAE.}
\label{fig:co2}
\end{figure*}
\begin{table}[t!]
\caption{Mean $\pm$ 1 standard deviation of normalized absolute errors in the experiments of section \ref{sct:exp}.}
\label{tab:bench}
\begin{center}
\begin{tabular}{lcccc}
\toprule
	 		& CO2 							& Airline \\
\midrule
$p$M-GP		& \textbf{0.52} $\pm$\textbf{0.70} 	& \textbf{0.52} $\pm$ \textbf{0.52} \\
PAs AR(2) 	& 0.77 $\pm$ 0.57 					& 0.94  $\pm$ 0.84 \\
BLR AR(2) 	& 1.31 $\pm$ 0.58					& 0.77  $\pm$ 0.62 \\
PAs AR(10) 	& 1.02 $\pm$ 0.57					& 0.92  $\pm$ 0.81 \\
BLR AR(10) 	& 1.77 $\pm$ 1.00					& 0.77  $\pm$ 0.53 \\
\bottomrule
\end{tabular}
\end{center}
\end{table}
In this section we start by demonstrating that the $p$M-GP filter provides considerably more accurate forecasts than competing fully-online approaches on standard real-life time series datasets. We then experimentally illustrate the sensitivity of the forecasting errors to the normalized aggressiveness parameter $c$, the trend, and the number of spectral components $K$.

\textbf{Benchmarking}: We compare our model to competing fully-online alternatives on the CO2 dataset of \cite{rasswill}, and the airline passengers dataset of \cite{box2015}. We select as competing benchmarks two autoregressive (AR) models, namely AR(2) and AR(10), and we use four different algorithms to learn the autoregressive coefficients online, namely the PA, PA-I and PA-II algorithms of \cite{pajmlr}, and Bayesian online linear regression with i.i.d. standard normal priors on the weights (BLR). For consistency with the $p$M-GP filter, we initialize the autoregressive weights in all four algorithms to $0$. We choose $c=100.0, \epsilon=0.0$ for $p$M-GP models and PA algorithms. We choose a linear trend and twice differentiability ($p=2$) for the $p$M-GP filter. BLR is run with a noise standard deviation of $5\%$ of the sample standard deviation of the corresponding time series. For each model we perform one-step ahead forecast. Mean absolute errors\footnote{Excluding the first forecast, which is the same for all models.} normalized by the standard deviation of increments (NMAE) are reported in Tab. \ref{tab:bench}, where we refer to all three PA algorithms as PAs because they perform identically up to two decimal points. The evolution of the running NMAE as a function of time is illustrated in Fig. \ref{fig:co2} (left) for the CO2 dataset. We note that our approach provides more accurate forecasts than competing alternatives.

\textbf{Sensitivity to Parameters}: The sensitivity of the accuracy of the $p$M-GP filter to the trend, the aggressiveness parameter $c$ and the number of spectral components is illustrated in Fig. \ref{fig:co2} (middle and right) for the CO2 dataset. Overall, for our data size (607 points), it can be seen that the error decreases as a function of $c$. This is in line with the empirical observation of \cite{pajmlr} (Fig. 5) that, for small datasets large $c$ perform better, and $c=0.001$ begins to outperform $c=100.0$ for the PA classification problems the author considered when the data size is in the thousands. Moreover, we note from Fig. \ref{fig:co2} (right) that for small datasets, a large $K$ should not necessarily be preferred to a smaller one, as more samples will typically be required to learn the model hyper-parameters.

\section{DISCUSSION}
\label{sct:con}
We propose an exact state space representation of Gaussian process regression for time series forecasting, under a family of kernels that is dense in the family of all stationary kernels, and that allows decoupling the flexibility of the covariance structure from the differentiability requirement of the latent time series. When hyper-parameters are known, exact GP predictive inference can be performed in constant time and with constant memory requirement. Critically, we propose a novel \textit{passive-aggressive} algorithm for online learning of the model hyper-parameters. The overall approach we refer to as the $p$M-GP filter has constant complexity and memory requirement, and provides more accurate forecasts than fully-online competing alternatives on the standard CO2 and airline passengers datasets. The approach may easily be extended to structured time series prediction, thus allowing for online learning of correlations between time series. Techniques that extend the Kalman filter to leptokurtic measurement noise models (e.g. \cite{aga11}) may also be used out-of-the box to robustify our approach.
\clearpage
%\printbibliography
\bibliography{pm-gp}
\newpage
\begin{appendices}
 \toptitlebar
\begin{center}{\centering \Large\bfseries Appendix}\end{center}
 \bottomtitlebar
\renewcommand\thesection{Appendix \Alph{section}}
Unless stated otherwise, the stochastic processes we consider throughout this appendix are indexed in $\mathbb{R}^+$. To ease notations, we use the superscript $^{(i)}$ to denote the $i$-th order derivative of a function or stochastic process when it exists, or the original function or stochastic process when $i=0$. In the case of stochastic processes, the derivative is to be understood in the mean square sense. Moreover, stationarity of stochastic processes is always meant in the weak sense. Furthermore, observation times are always assumed to be distinct and sorted by index: $t_0 < \dots < t_k < \dots$.

\section{}
\label{app:hyp_disc}
\textbf{Objective:} In this section we  discuss the appropriateness of assuming that the latent time series is a trend-stationary Gaussian process. More precisely, we argue that given \textit{a single path} of a time series \textit{on some bounded interval} $[0, T]$, neither the trend-stationarity assumption nor the Gaussianity assumption can be invalidated.

\textbf{Trend-Stationarity:} Firstly, we note that unless further assumptions are made about the trend other than it being smooth, trend-stationarity cannot be invalidated using discrete observations of a \textit{single path} as we can always find an infinitely smooth trend or mean function $\hat{m}$ (using polynomials for instance), that coincides with all discrete observations, making the observed path highly likely to result from any trend-stationary stochastic process with mean function $\hat{m}$.

In practice however, the trend might be assumed to lie in a parametric family. Even in that case, the stationarity of the residual time series can hardly be invalidated. In effect, most stationarity statistical tests have as null hypothesis that the sample comes from a nonstationary time series. Evidence is then gathered from the sample through the test statistics to determine whether the null hypothesis can be rejected with some confidence, or equivalently if there is enough evidence in the sample to conclude stationarity. Hence, fundamentally, such an approach cannot be used to conclude nonstationarity, as the lack of evidence of stationarity in a \textit{given sample} is not an evidence of nonstationarity of the \textit{underlying process}. It might well be that the sample does not contain enough information to fully characterise the underlying stochastic process so that, had we collected more data, we might have been able to conclude stationarity. It is in the same spirit that \cite{kpss92} noted that `most economic time series are not very informative about whether or not there is a unit root'. As we do not assume that we have enough data to characterize the latent process, our trend-stationarity assumption cannot be invalidated experimentally. 

To illustrate our point, we drew a path from a stationary Gaussian process with mean $0$ and squared exponential covariance function $k(u, v) = e^{-\frac{1}{2}(u-v)^2}$ on $[0, 10]$ discretized with mesh size $0.001$ (see Figure \ref{fig:sgp_is_fine}).
\begin{figure}[h]
\includegraphics[width=0.48\textwidth]{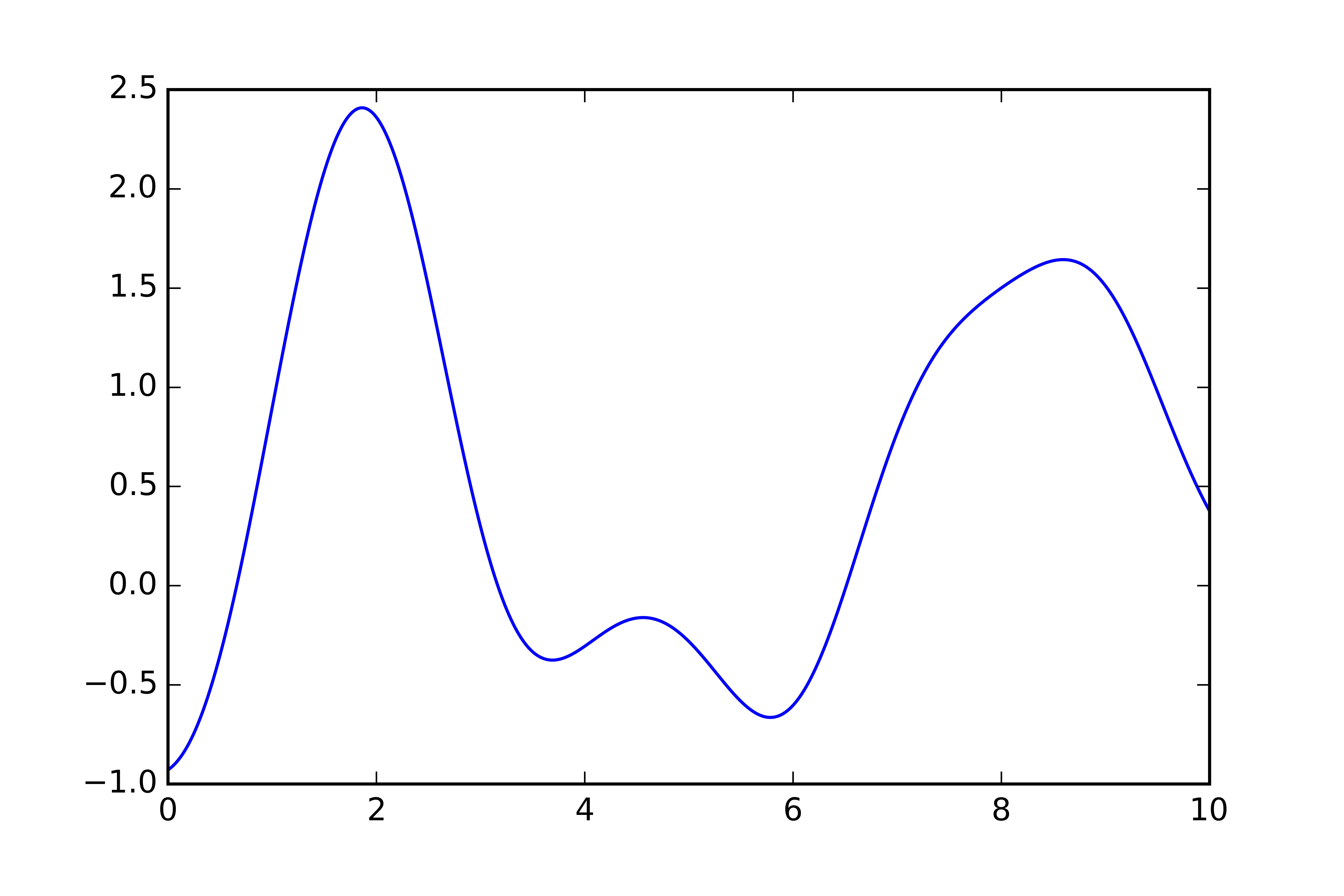}
\caption{Draw from a stationary GP on $[0, 10]$.}
\label{fig:sgp_is_fine}
\end{figure}
We ran three standard stationarity statistical tests on the sample, namely the Augmented Dickey-Fuller test (ADF),  the ADF-GLS test and the Phillips-Perron test.
\begin{table}[h]
\caption{Results of stationarity tests on the time series in Figure \ref{fig:sgp_is_fine}.} 
\label{tbl:sgp_is_fine}
\begin{center}
\begin{tabular}{llll}
Test  			& Statistics 	& p-Value 	& Lags \\
\hline \\
ADF         		&-2.12 	&0.24		&22\\
ADF-GLS       		&-1.02 	&0.28		&22\\
Phillips-Perron          &-1.64 	&0.46		&22
\end{tabular}
\end{center}
\end{table}
As can be seen in Table \ref{tbl:sgp_is_fine}, all three tests failed to find evidence for (to conclude) stationarity with $80\%$ confidence, despite the sample coming from a stationary stochastic process. As all three stationarity tests rely on ergodicity for the first two moments, one might be tempted to think that this could be an indication that the underlying stochastic process is not ergodic. However, a sufficient condition for a mean zero stationary Gaussian process to be ergodic for the first two moments is that its covariance function vanishes as the lag increases \cite[see][\S 13.1]{papoulis02}
\begin{align}
\forall t, ~k(\tau) := \text{cov}(z_{t}, z_{t+\tau}) \underset{\tau \to +\infty}{\longrightarrow}0,\nonumber
\end{align}
and this condition is satisfied by the squared exponential kernel. The real issue at play here is that, given a \textit{finite sample} of a time series, it is hardly possible to say with confidence that it comes from a nonstationary time series.

\textbf{Gaussianity:} Similarly, with one single realisation, it is hardly possible to determine whether the sample comes from a multivariate Gaussian, and thus the Gaussian process assumption cannot be invalidated either. Testing whether a real-valued vector is a draw from a multivariate Gaussian without any assumption on its mean and covariance matrix is as hopeless as testing whether a real-value scalar is a draw from a univariate Gaussian random variable without any assumption on its mean or variance.

\section{}
\label{sct:only_smooth_mgp_}
\textbf{Objective:} In this section we  prove that if a real-valued trend-stationary Gaussian process with a differentiable mean function is mean square differentiable and Markovian, then it has a constant covariance function (or equivalently it is almost surely equal to its mean function plus a constant).

\begin{proof} Let us consider a real-valued trend-stationary, mean square differentiable and Markovian Gaussian process $(z_t)_{t \geq 0}$. Let us denote $(u ,v) \to k(u-v)$ its covariance function. It follows that \begin{equation}
\label{eq:mkv_sm}
\forall t, h, ~\text{cov}\left(z_{t+h}, z_{t-h} \vert z_t \right) = k(2h) - \frac{k(h)^2}{k(0)} = 0,
\end{equation}
where the first equality results from standard Gaussian identities and the second equality results from $(z_t)_{t \geq 0}$ being Markovian. As $(z_t)_{t \geq 0}$ is also assumed to be mean square differentiable and to have a differentiable mean function, the centred Gaussian process $(z_t - \text{E}(z_t))_{t \geq 0}$ is mean square differentiable, or equivalently $k$ is twice differentiable everywhere. Hence, $h \to  k(2h) - \frac{k(h)^2}{k(0)}$ is also twice differentiable at $0$ and has second order derivative $2k^{(2)}(0)$ at $0$. It then follows from Equation (\ref{eq:mkv_sm}) that $k^{(2)}(0) =0$. As $h \to -k^{(2)}(h)$ is the covariance function of $(z^{(1)}_t)_{t \geq 0}$, we have:
\begin{align}
k^{(2)}(h)^2= \text{cov}\left(z^{(1)}_t, z^{(1)}_{t+h}\right)^2 &\leq \text{var}\left(z^{(1)}_t\right) \text{var}\left(z^{(1)}_{t+h}\right)\nonumber \\
 &\leq k^{(2)}\left(0\right)^2 = 0.\nonumber
\end{align}
That is, $\forall h, k^{(2)}\left(h\right)=0$, or equivalently $k$ is a linear function of $h$. Moreover, as $(z_t)_{t \geq 0}$ is trend-stationary $k$ is also bounded\footnote{More precisely, the positive-definiteness of the covariance matrix between $z_t$ and $z_{t+h}$ implies $\forall h, \vert k(h) \vert \leq k(0)$.}. Hence, $k$ is a constant function. This means that for any times $u$ and $v$, $z_u - \text{E}(z_u)$ and $z_v - \text{E}(z_v)$ have the same mean, the same variance and correlation $1$, which implies 
\begin{equation*}
\forall u, v \geq  0, ~ z_u - \text{E}(z_u) \overset{\text{a.s.}}{=} z_v - \text{E}(z_v). 
\end{equation*}
\end{proof}
\section{}
\label{sct:pmat_lgssm}
\textbf{Objective:} In this section we derive the state space representation of the Gaussian process regression model under a trend-stationary latent GP $(z_t)_{t \geq 0}$, with mean function $m$, Mat\'{e}rn-($p+\frac{1}{2}$) kernel, and with a Gaussian white noise $(\epsilon_t)_{t \geq 0}$.
 
\textbf{Derivation:} If we denote $f_t = z_t - m(t)$, $(x_t)_{t \geq 0}$ with $x_t = \left(f_t, f_t^{(1)}, \dots, f_t^{(p)}\right)$ the $p$DGP corresponding to $(f_t)_{t \geq 0}$, and $H=(1, 0, \dots, 0)$, it is easy to see that if we use $x_t$ as state variable, the measurement equation of the state space model should be
\[\forall t, ~y_t = m(t) + H^Tx_t + \epsilon_t.\]
Denoting $K_{u, v}$ the cross-covariance matrix between $x_u$ and $x_v$,  $K_{u \vert v} = L_{u \vert v}L_{u \vert v}^T$ the auto-covariance matrix of $x_u$ conditional on $x_v$, and $t_0$ the initial time, we get the initial state distribution:
\[x_{t_0} \sim \mathcal{N}(0, K_{t_0, t_0}).\]
For observations times $t_0, \dots, t_T$, using Bayes' rule and Corollary \ref{coro:coro}, we get
\begin{align}
p\left(x_{t_0}, \dots, x_{t_T}\right) &= p(x_{t_0}) \prod_{k=1}^{T}p\left(x_{t_k} | x_{t_0:t_{k-1}}\right) \nonumber\\
&= p(x_{t_0}) \prod_{k=1}^{T}p\left(x_{t_k} | x_{t_{k-1}}\right). \nonumber
\end{align}
Moreover, we note that
\[x_{t_k} | x_{t_{k-1}}  \sim \mathcal{N}\left(F_{t_k}x_{t_{k-1}}, K_{t_k|t_{k-1}}\right) \]
with $F_{t_k}=K_{t_k, t_{k-1}} K_{t_{k-1}, t_{k-1}}^{-1}$, and we recall that if $X$ is a vector of $(p+1)$ i.i.d. standard normal, $M$ a deterministic vector and $L$ a square matrix both with appropriate dimensions, then $M + LX$ is a Gaussian vector with mean $M$ and covariance matrix $LL^T$. It then follows that the full state space representation reads

\begin{equation*}
\begin{cases}
x_{t_0} \sim \mathcal{N}(0, K_{t_0, t_0}) \\
x_{t_k} = F_{t_k} x_{t_{k-1}} + L_{t_k \vert t_{k-1}}\xi_{t_k}\\
y_{t_k} = m(t_k) + H^Tx_{t_k} + \epsilon_{t_k}
\end{cases}
\end{equation*}
where $(\xi_{t})_{t \geq 0}$ is a $(p+1)$-dimensional standard Gaussian white noise.

\section{}
\label{sct:pm_gp_fil}
\textbf{Objective:} In this section we prove that the $p$M-GP filter is equivalent to Gaussian process regression under a spectral Mat\'{e}rn kernel (Proposition \ref{prop:eqv}).

\begin{proof} Following the notations of Definition \ref{def:pm-gpf} and Proposition \ref{prop:eqv}, and using the result of \ref{sct:pmat_lgssm}, we first note that by construction the processes $\{ \dots, ({}^i_{c}x_{t})_{t \geq 0}, ({}^i_{s}x_{t})_{t \geq 0}, \dots \}$ are mutually independent and both $({}^i_{c}x_{t})_{t \geq 0}$ and $({}^i_{s}x_{t})_{t \geq 0}$ are $p$DGP with mean $0$ and Mat\'{e}rn covariance function $k_{\text{ma}}(\tau; k_{0i}, l_i, p + \frac{1}{2})$. Writting \[{}^i_{*}x_{t} = \left({}^i_{*}z_{t}, {}^i_{*}z_{t}^{(1)}, \dots, {}^i_{*}z_{t}^{(p)}\right)\] where $*$ is either $c$ or $s$, it is easy to see that the observations process of the $p$M-GP filter can be written down as
\[y_t = z_t + \epsilon_t, \] 
where
\begin{equation}
\label{eq:latent}
z_t = m(t) + \sum_{i=0}^n {}^i_{c}z_t \cos(\omega_i t) + {}^i_{s}z_t \sin(\omega_i t).
\end{equation}
It is also easy to see that $(z_t)_{t \geq 0}$ is a Gaussian process  with mean function $m$, and by mutual independence of $\{ \dots, ({}^i_{c}x_{t})_{t \geq 0}, ({}^i_{s}x_{t})_{t \geq 0}, \dots \}$ we also have that
\begin{align}
\text{cov}(z_u, z_v) &= \sum_{i=0}^n \bigg(\text{cov}({}^i_{c}z_u, {}^i_{c}z_v) \cos(\omega_i u)\cos(\omega_i v) \nonumber \\
&~~~~~~+ \text{cov}({}^i_{s}z_u, {}^i_{s}z_v) \sin(\omega_i u)\sin(\omega_i v) \bigg) \nonumber \\
&= \sum_{i=0}^n k_{\text{ma}}\left(\tau; k_{0i}, l_i, p + \frac{1}{2}\right) \cos\left(\omega_i (u-v)\right). \nonumber
\end{align}
This proves that $(z_t)_{t \geq 0}$ has the same law as $(\hat{z}_t)_{t \geq 0}$. Given that  $(\epsilon_t)_{t \geq 0}$ has the same law as $(\hat{\epsilon}_t)_{t \geq 0}$, it follows that $(y_t)_{t \geq 0}$ has the same law as $(\hat{y}_t)_{t \geq 0}$, which concludes the proof.
\end{proof}

\section{}
\label{sct:solu_fill}
\textbf{Objective:} In this section we derive the solution to the forecasting problem for the $p$M-GP filter (Equations (\ref{eq:solu2}, \ref{eq:pred}, \ref{eq:update})); in particular we provide an iterative algorithm to compute the posterior distribution over future values of the latent time series given historical noisy observations: $p(z_{t} \vert y_{t_k} \dots y_{t_0})$ for $0 < t_0< \dots < t_k < t$.

\textbf{Derivation:} The derivation is almost identical to the Bayesian derivation of the Kalman filter, except for the additional trend term $m$ in the observation equation. 

We note from Equation (\ref{eq:solu}) that the processes $(y_t)_{t \geq 0}$ and $(\bm{x}_t)_{t \geq 0}$ are jointly Gaussian. Hence, $\bm{x}_{t_0}$ and $y_{t_0}$ are jointly Gaussian, which implies $\bm{x}_{t_0} \vert y_{t_0}$ is Gaussian too. Using Equation (\ref{eq:solu}) and Proposition \ref{prop:eqv} we get:
\begin{align}
\forall t > 0, ~\text{E}(\bm{x}_t) = 0, ~ \text{E}(y_t) = m(t), ~ \text{E}(z_t) = m(t) \nonumber
\end{align}
\begin{flalign}
\text{cov}\left(\bm{x}_t, y_t \right) &:= \text{E}\left(\bm{x}_ty_t \right) - \text{E}\left(\bm{x}_t\right) \text{E}\left(y_t\right) \nonumber & \\
&= m(t) \text{E}\left(\bm{x}_t\right) + \text{E}\left(\bm{x}_t\bm{H}_t^T\bm{x}_t\right) + \text{E}\left(\bm{x}_t\epsilon_t\right) \nonumber & \\
&= \bm{K}_{t, t} \bm{H}_t\nonumber &
\end{flalign}
\begin{flalign}
&\text{cov}\left(\bm{x}_t, \bm{x}_t \right) := \bm{K}_{t,t}\nonumber &\\
&\text{var}(y_t) = \bm{H}_t^T\bm{K}_{t,t} \bm{H}_t + \sigma^2\nonumber &\\
&\text{cov}\left(y_t, z_t \right)  = \text{cov}\left(z_t, z_t \right) = \bm{H}_t^T\bm{K}_{t,t} \bm{H}_t. & \nonumber
\end{flalign}
Moreover, using standard Gaussian identities, we get:
\begin{align}
&\text{E}\left(\bm{x}_t \vert y_t \right) =\frac{1}{\bm{H}_t^T\bm{K}_{t,t} \bm{H}_t + \sigma^2} \bm{K}_{t, t} \bm{H}_t \left(y_t-m(t) \right)\nonumber\\
&\text{cov}\left(\bm{x}_t \vert y_t \right) = \bm{K}_{t,t} - \frac{1}{\bm{H}_t^T\bm{K}_{t,t} \bm{H}_t + \sigma^2} \bm{K}_{t, t} \bm{H}_t\bm{H}_t^T\bm{K}_{t, t}^T \nonumber
\end{align}
\begin{align}
& \text{E}\left(z_t \vert y_t\right) \nonumber \\
&= m(t) + \frac{1}{\bm{H}_t^T\bm{K}_{t,t} \bm{H}_t + \sigma^2}\bm{H}_t^T\bm{K}_{t,t} \bm{H}_t \left(y_t-m(t) \right)\nonumber\\
&= m(t) + \bm{H}_t^T\text{E}\left(\bm{x}_t \vert y_t \right) \nonumber
\end{align}
\begin{align}
\text{cov}\left(z_t \vert y_t \right) &=  \bm{H}_t^T\bm{K}_{t,t} \bm{H}_t - \frac{ \bm{H}_t^T\bm{K}_{t,t} \bm{H}_t \bm{H}_t^T\bm{K}_{t,t}^T\bm{H}_t}{\bm{H}_t^T\bm{K}_{t,t} \bm{H}_t + \sigma^2}  \nonumber \\
&=\bm{H}_t^T \text{cov}\left(\bm{x}_t \vert y_t \right) \bm{H}_t. \nonumber
\end{align}
Hence, with 
\begin{align}
\bm{m}_{t_0}^{-} &= 0 \nonumber \\
\bm{P}_{t_0}^{-} &= \bm{K}_{t_0, t_0} \nonumber
\end{align}
and
\begin{align}
\label{eq:solu_def}
\forall t
\begin{cases}
v_t^{-} &:= \bm{H}_t^T \bm{P}_{t}^{-} \bm{H}_t + \sigma^2 \\
\bm{e}_t^{-} &:= y_t - m(t) - \bm{H}_t^T\bm{m}_t^{-} \\
\bm{G}_t &:= \frac{1}{v_t^{-}}\bm{P}_t^{-}\bm{H}_t \\
\bm{m}_t &:= \bm{m}_t^{-} + \bm{e}_t^{-}\bm{G}_t \\
\bm{P}_t &:= \bm{P}_t^{-} - v_t^{-}\bm{G}_t\bm{G}_t^T\\
v_t &:= \bm{H}_t^T\bm{P}_t\bm{H}_t
\end{cases},
\end{align}
we get 
\begin{align}
\bm{x}_{t_0} \vert  y_{t_0} &\sim \mathcal{N}(\bm{m}_{t_0}, \bm{P}_{t_0}) \nonumber \\
z_{t_0} \vert  y_{t_0} &\sim \mathcal{N}(m(t) + \bm{H}_{t_0}^T\bm{m}_{t_0}, v_{t_0}). \nonumber
\end{align}
In order to derive the remaining steps, we need the following lemma \cite[see][Lemma 6, and Appendix G for the proof]{samo_sgp}.
\begin{lemma}
\label{lem:gaussian_message}
Let $X$ be a multivariate Gaussian with mean $\mu_X$ and covariance matrix $\Sigma_X$. If conditional on $X$, $Y$ is a multivariate Gaussian with mean $MX + A$  and covariance matrix $\Sigma_Y^c$ where $M$, $A$ and $\Sigma_Y^c$ do not depend on $X$, then $(X, Y)$ is a jointly Gaussian vector with mean \[\mu_{X;Y}=\begin{bmatrix} \mu_X \\ M\mu_X + A \end{bmatrix}\]and covariance matrix
\[\Sigma_{X;Y}=\begin{bmatrix} \Sigma_X & \Sigma_X M^T \\ M\Sigma_X & \Sigma_Y^c + M \Sigma_X M^T\end{bmatrix}.\]
\end{lemma}
For $k>0$ and $t>t_{k-1}$ we proceed by iteration. We assume that 
\begin{align}
\label{eq:rec_ass}
\begin{cases}
\bm{x}_{t_{k-1}} \vert  y_{t_0:t_{k-1}} &\sim \mathcal{N}(\bm{m}_{t_{k-1}}, \bm{P}_{t_{k-1}}) \\
z_{t_{k-1}} \vert  y_{t_0:t_{k-1}} &\sim \mathcal{N}(m(t_{k-1}) + \bm{H}_{t_{k-1}}^T\bm{m}_{t_{k-1}}, v_{t_{k-1}})
\end{cases}
\end{align}
using the definitions in Equations (\ref{eq:solu_def}), and we would like to prove that 
\begin{align}
\label{eq:rec_impl}
\begin{cases}
\bm{x}_t \vert y_{t_0 : t_{k-1}}  &\sim  \mathcal{N}(\bm{m}_t^{-}, \bm{P}_t^{-})\\
y_t \vert y_{t_0 : t_{k-1}}  &\sim  \mathcal{N}(m(t) + \bm{H}_t^T\bm{m}_t^{-}, v_t^{-})\\
z_t \vert y_{t_0 : t_{k-1}}  &\sim  \mathcal{N}(m(t) + \bm{H}_t^T\bm{m}_t^{-},v_t^{-}-\sigma^2 ),
\end{cases}
\end{align}
with
\begin{align}
\begin{cases}
\bm{m}_t^{-} &= \bm{F}_t \bm{m}_{t_{k-1}}\nonumber \\
\bm{P}_t^{-} &= \bm{F}_t \bm{P}_{t_{k-1}} \bm{F}_t^T + \bm{K}_{t \vert t_{k-1}} \nonumber.
\end{cases}
\end{align}
To do so, we note that
\begin{align}
&p\left(\bm{x}_t, \bm{x}_{t_{k-1}} \vert y_{t_0 : t_{k-1}}\right) \nonumber \\
&= p\left(\bm{x}_t \vert \bm{x}_{t_{k-1}}, y_{t_0 : t_{k-1}}\right)p\left(\bm{x}_{t_{k-1}} \vert y_{t_0 : t_{k-1}}\right) \nonumber\\
&=  p\left(\bm{x}_t \vert \bm{x}_{t_{k-1}}\right)p\left(\bm{x}_{t_{k-1}} \vert y_{t_0 : t_{k-1}}\right) \nonumber,
\end{align}
where the first equality results from Bayes' rule and the second result from the Markov property of $(\bm{x}_t)_{t \geq 0}$. As 
\begin{align}
\bm{x}_t \vert \bm{x}_{t_{k-1}}, y_{t_0 : t_{k-1}} &\sim \mathcal{N}\left(\bm{F}_t \bm{x}_{t_{k-1}}, \bm{K}_{t \vert t_{k-1}}\right)\nonumber\\
\bm{x}_{t_{k-1}} \vert  y_{t_0:t_{k-1}} &\sim \mathcal{N}(\bm{m}_{t_{k-1}}, \bm{P}_{t_{k-1}})\nonumber,
\end{align}
it follows from Lemma \ref{lem:gaussian_message} that
\begin{align}
\bm{x}_t \vert y_{t_0 : t_{k-1}}  &\sim  \mathcal{N}(\bm{F}_t\bm{m}_{t_{k-1}}, \bm{K}_{t \vert t_{k-1}} + \bm{F}_t\bm{P}_{t_{k-1}}\bm{F}_t^T)\nonumber \\
&\sim  \mathcal{N}(\bm{m}_t^{-}, \bm{P}_t^{-})\nonumber.
\end{align}
Moreover, as $(y_t)_{t \geq 0}$ is a Gaussian process, $y_t \vert y_{t_0 : t_{k-1}}$ is Gaussian. What's more, as $y_t = m(t) + \bm{H}_{t}^T\bm{x}_t + \epsilon_t$ and $\epsilon_t \perp y_{t_0 : t_{k-1}}$, we have that:
\begin{align}
\text{E}(y_t \vert y_{t_0 : t_{k-1}}) &= m(t) + \bm{H}_{t}^T\text{E}\left(\bm{x}_t \vert y_{t_0 : t_{k-1}}\right) \nonumber \\
&= m(t) + \bm{H}_{t}^T\bm{m}_t^{-}\nonumber
\end{align}
and 
\begin{align}
\text{var}(y_t \vert y_{t_0 : t_{k-1}}) &= \bm{H}_{t}^T\text{cov}\left(\bm{x}_t, \bm{x}_t \vert y_{t_0 : t_{k-1}}\right)\bm{H}_{t} + \text{E}(\epsilon_t^2) \nonumber \\
&= \bm{H}_{t}^T\bm{P}_t^{-}\bm{H}_{t} + \sigma^2 \nonumber\\
&= v_t^{-}\nonumber.
\end{align}
The distribution $z_t \vert y_{t_0 : t_{k-1}}$ is obtained in a similar fashion. 

More generally, as the processes $(\bm{x}_t)_{t \geq 0}$ and $(y_t)_{t \geq 0}$ are jointly Gaussian, the random variables $\bm{x}_t$ and $y_t$ are also jointly Gaussian conditional on $y_{t_0 : t_{k-1}}$ and 
\begin{flalign}
&\text{cov}(\bm{x}_t, y_t \vert y_{t_0 : t_{k-1}}) \nonumber& \\
&= \text{cov}\left(\bm{x}_t, m(t) +  \bm{H}_{t}^T\bm{x}_t + \epsilon_t \vert y_{t_0 : t_{k-1}}\right)\nonumber& \\
&= \text{cov}\left(\bm{x}_t, \bm{x}_t \vert y_{t_0 : t_{k-1}}\right)\bm{H}_{t} + \text{E}(\bm{x}_t\epsilon_t \vert y_{t_0 : t_{k-1}}) \nonumber& \\
&= \bm{P}_t^{-}\bm{H}_{t} + \text{E}(\bm{x}_t \vert y_{t_0 : t_{k-1}})\text{E}(\epsilon_t) \nonumber&\\
&=\bm{P}_t^{-}\bm{H}_{t}.\nonumber&
\end{flalign}
Finally, under the assumptions of Equations (\ref{eq:rec_ass}), which we recall imply Equations (\ref{eq:rec_impl}), we would like to prove that 
\begin{align}
\label{eq:rec_impl2}
\begin{cases}
\bm{x}_{t} \vert  y_{t_0:t} &\sim \mathcal{N}(\bm{m}_t, \bm{P}_t) \\
z_{t} \vert  y_{t_0:t} &\sim \mathcal{N}(m(t) + \bm{H}_t^T\bm{m}_t, v_t)
\end{cases}.
\end{align}
We have previously established that $\bm{x}_t, y_t \vert y_{t_0 : t_{k-1}}$ is Gaussian and we have derived the corresponding mean and covariance matrix. Noting that by definition \[\bm{x}_t\vert (y_{t_0 : t_{k-1}}, y_t)  := \bm{x}_t\vert y_{t_0 : t},\]
it follows from standard Gaussian identities that
\begin{flalign}
&\text{E}\left( \bm{x}_t\vert y_{t_0 : t} \right) \nonumber & \\
&= \text{E}\left( \bm{x}_t\vert y_{t_0 : t_{k-1}} \right) + \frac{y_t - \text{E}\left( y_t\vert y_{t_0 : t_{k-1}} \right)}{\text{var}\left( y_t \vert  y_{t_0 : t_{k-1}}\right)} \nonumber & \\ 
&~~~~\times \text{cov}\left( \bm{x}_t, y_t\vert y_{t_0 : t_{k-1}} \right)  \nonumber & \\
&= \bm{m}_t^{-} + \frac{1}{v_t^{-}} \bm{P}_t^{-}\bm{H}_{t} \left(y_t -m(t) - \bm{H}_{t}^T\bm{m}_t^{-}\right) \nonumber& \\
& = \bm{m}_t^{-} +  \bm{G}_t \bm{e}_t^{-} \nonumber& \\
&= \bm{m}_t \nonumber&
\end{flalign}
and 
\begin{flalign}
&\text{cov}\left(\bm{x}_t, \bm{x}_t\vert y_{t_0 : t} \right) \nonumber& \\
&=\text{cov}\left(\bm{x}_t, \bm{x}_t\vert y_{t_0 : t_{k-1}} \right) \nonumber& \\
&- \frac{ \text{cov}\left( \bm{x}_t, y_t\vert y_{t_0 : t_{k-1}} \right)  \text{cov}\left( \bm{x}_t, y_t\vert y_{t_0 : t_{k-1}} \right)^T}{\text{var}\left( y_t \vert y_{t_0 : t_{k-1}}\right)} \nonumber& \\
&= \bm{P}_t^{-} -\frac{1}{v_t^{-}}\bm{P}_t^{-}\bm{H}_{t} \bm{H}_{t}^T \bm{P}_t^{-T} \nonumber& \\
&= \bm{P}_t^{-} - v_t^{-} \bm{G}_t \bm{G}_t^T \nonumber& \\
&= \bm{P}_t.&
\end{flalign}
This proves the first part of Equations (\ref{eq:rec_impl2}). As for the second part, it is a direct consequence of \[z_t = m(t) + \bm{H}_t^T \bm{x}_t.\]

\section{}
\label{sct:optim_solu}
\textbf{Objective:}  In this section we derive the solution to the constrained optimization problem (\ref{prob:2}).

\begin{proof}
We start by recalling the problem of interest:
\begin{align}
\begin{cases}
\bm{\theta}_{t_k} = \underset{\bm{\theta}}{\text{argmin}} ~|| \bm{\theta} - \bm{\theta}_{t_{k-1}}||^2 + c_k \xi^2\\
\text{ s.t. } \max\left(-\epsilon -\hat{\mathcal{L}}^{l}_{t_k}(\bm{\theta}), 0\right) \leq \xi
\end{cases},\nonumber
\end{align}
with $\hat{\mathcal{L}}^{l}_{t_k}(\bm{\theta}) = \mathcal{L}^{l}_{t_k}(\bm{\theta}_{t_{k-1}}) + \nabla \mathcal{L}^{l}_{t_k}(\bm{\theta}_{t_{k-1}})^T (\bm{\theta} - \bm{\theta}_{t_{k-1}})$, $c_k>0$ and $\epsilon \geq 0$.

It is easy to note that when $\mathcal{L}^{l}_{t_k}(\bm{\theta}_{t_{k-1}}) > -\epsilon$, the solution to the optimization problem is $(\bm{\theta}_{t_{k-1}}, 0)$, so that we may focus on the case $\mathcal{L}^{l}_{t_k}(\bm{\theta}_{t_{k-1}}) \leq -\epsilon$. 

For  $\mathcal{L}^{l}_{t_k}(\bm{\theta}_{t_{k-1}}) \leq -\epsilon$, the problem can then be rewritten as the convex optimization problem:
\begin{align}
\begin{cases}
\bm{\theta}_{t_k} = \underset{\bm{\theta}}{\text{argmin}} ~|| \bm{\theta} - \bm{\theta}_{t_{k-1}}||^2 + c_k \xi^2\\
\text{s.t.} -\epsilon -\mathcal{L}^{l}_{t_k}(\bm{\theta}_{t_{k-1}}) - \nabla \mathcal{L}^{l}_{t_k}(\bm{\theta}_{t_{k-1}})^T (\bm{\theta} - \bm{\theta}_{t_{k-1}}) -\xi \leq 0.
\end{cases}\nonumber
\end{align}
As the constraints are linear and the domain of the objective is not restricted, Slater's condition is met and strong duality holds \cite[see][\S 5.2.3]{boyd2004convex}. The minimizer is therefore obtained by setting the gradient of the Lagrangian 
\begin{align}
\label{eq:lagr}
&|| \bm{\theta} - \bm{\theta}_{t_{k-1}}||^2 + c_k \xi^2 + \lambda \bigg(  -\epsilon -\mathcal{L}^{l}_{t_k}(\bm{\theta}_{t_{k-1}}) -\xi \nonumber \\
& - \nabla \mathcal{L}^{l}_{t_k}(\bm{\theta}_{t_{k-1}})^T (\bm{\theta} - \bm{\theta}_{t_{k-1}})  \bigg), \lambda \geq 0
\end{align}
to $0$. Setting the gradient with respect to $\bm{\theta}$ to $0$ we get
%
%s 
\begin{equation}
\label{eq:dr_th}
\bm{\theta}_{t_k} = \bm{\theta}_{t_{k-1}} + \frac{\lambda^*}{2} \nabla \mathcal{L}^{l}_{t_k}(\bm{\theta}_{t_{k-1}}).
\end{equation}
Setting the derivative with respect to $\xi$ to zero, we get 
\begin{equation}
\label{eq:dr_xi}
\xi^* = \frac{\lambda^*}{2c_k}.
\end{equation}
Finally, plugging Equations (\ref{eq:dr_th}) and (\ref{eq:dr_xi}) into Equation (\ref{eq:lagr}), we can rewrite the Lagrangian as
\begin{equation*}
-\left( \frac{\|  \nabla \mathcal{L}^{l}_{t_k}(\bm{\theta}_{t_{k-1}})\|^2}{4} + \frac{1}{4c}\right)(\lambda^*)^2 - \left(\epsilon + \mathcal{L}^{l}_{t_k}(\bm{\theta}_{t_{k-1}})\right)\lambda^*,
\end{equation*}
which reaches its maximum at 
\begin{equation}
\label{eq:dr_lm}
\lambda^* = -2c_k\frac{\epsilon + \mathcal{L}^{l}_{t_k}(\bm{\theta}_{t_{k-1}})}{1 + c_k\|  \nabla \mathcal{L}^{l}_{t_k}(\bm{\theta}_{t_{k-1}})\|^2}.
\end{equation}
Using Equations (\ref{eq:dr_th}) and (\ref{eq:dr_lm}), together with the result established for the case $\mathcal{L}^{l}_{t_k}(\bm{\theta}_{t_{k-1}}) > -\epsilon$, we conclude that 
\begin{equation*}
\bm{\theta}_{t_k} =  \bm{\theta}_{t_{k-1}} + c_k \frac{\max\left(-\epsilon - \mathcal{L}^{l}_{t_k}(\bm{\theta}_{t_{k-1}}), 0\right)}{1 + c_k\|  \nabla \mathcal{L}^{l}_{t_k}(\bm{\theta}_{t_{k-1}})\|^2} \nabla \mathcal{L}^{l}_{t_k}(\bm{\theta}_{t_{k-1}}),
\end{equation*}
which ends the proof.
\end{proof}
\section{}
\label{sct:optim_upd}
\textbf{Objective:} In this section we derive  $\nabla \mathcal{L}^{l}_{t_k}(\bm{\theta})$. We assume that the trend function $m$ is parametric and has parameters $\bm{\beta}$.

\textbf{Derivation:} We recall that $\mathcal{L}^{l}_{t_k}(\bm{\theta})$ is the logarithm of the probability density function of a Gaussian. To ease derivations we denote \begin{equation*}
\bar{m}_{t_k}(\bm{\theta}) :=  m(t_k) + \bm{H}_{t_k}^T\bm{m}_{t_k}^-  \text{ and } \bar{v}_{t_k}(\bm{\theta}) := v_{t_k}^-
\end{equation*}
the mean and variance of the corresponding Gaussian, so that
\begin{equation}
\mathcal{L}^{l}_{t_k}(\bm{\theta}) = -\frac{\log(2\pi)}{2}  - \frac{\log\left(\bar{v}_{t_k}(\bm{\theta})\right)}{2}  - \frac{\left(y_{t_k} - \bar{m}_{t_k}(\bm{\theta}) \right)^2}{2\bar{v}_{t_k}(\bm{\theta})}.
\end{equation}
It then follows that:
\begin{align}
\nabla \mathcal{L}^{l}_{t_k}(\bm{\theta}) &= \left(- \frac{1}{2\bar{v}_{t_k}(\bm{\theta})} + \frac{\left(y_{t_k} - \bar{m}_{t_k}(\bm{\theta}) \right)^2}{2\bar{v}_{t_k}(\bm{\theta})^2}\right) \nabla \bar{v}_{t_k}(\bm{\theta}) \nonumber \\
&+ \left(\frac{y_{t_k} - \bar{m}_{t_k}(\bm{\theta})}{\bar{v}_{t_k}(\bm{\theta})}\right) \nabla \bar{m}_{t_k}(\bm{\theta}),
\end{align}
so that all we need to do is derive the gradients of $\bar{m}_{t_k}$ and $\bar{v}_{t_k}$ (using Equations (\ref{eq:pred}) and (\ref{eq:update})). 

To do so, we recall that $\bm{\theta}$ is made of the parameters of $m$ that we denote $\bm{\beta}$, $\log \sigma$ and $\{\log k_{0i}, \log l_i, \log \omega_i\}_{i=0}^n$.

\textbf{Derivatives with respect to $\bm{\beta}$:}
\begin{flalign}
& \frac{\partial \bar{m}_{t_k}(\bm{\theta})}{\partial \bm{\beta}} = \frac{\partial m}{\partial \bm{\beta}},  ~\frac{\partial \bar{v}_{t_k}(\bm{\theta})}{\partial \bm{\beta}} =  0. \nonumber &
\end{flalign}
\textbf{Derivatives with respect to $\log \sigma$:}
\begin{flalign}
& \frac{\partial \bar{m}_{t_k}(\bm{\theta})}{\partial \log \sigma} = 0 , ~ \frac{\partial \bar{v}_{t_k}(\bm{\theta})}{\partial \log \sigma} = \frac{\partial \bar{v}_{t_k}(\bm{\theta})}{\partial \sigma^2} \frac{\text{d} \sigma^2}{\text{d} \log \sigma} = 2\sigma^2 \nonumber. &
\end{flalign}
\textbf{Derivatives with respect to $\log \omega_i$:}
\begin{flalign}
&\frac{\partial \bar{m}_{t_k}(\bm{\theta})}{\partial \log \omega_i} =  \frac{\partial \bm{H}_{t_k}^T}{\partial \log \omega_i} \bm{m}_{t_k}^- \nonumber & \\
&\frac{\partial \bar{v}_{t_k}(\bm{\theta})}{\partial \log \omega_i} =  \frac{\partial \bm{H}_{t_k}^T}{\partial \log \omega_i} \bm{P}_{t_k}^- \bm{H}_{t_k} +  \bm{H}_{t_k}^T \bm{P}_{t_k}^-  \frac{\partial \bm{H}_{t_k}}{\partial \log \omega_i}\nonumber&
\end{flalign}
where $\frac{\partial \bm{H}_{t_k}^T}{\partial \log \omega_i}$ is identical to $\bm{H}_{t_k}^T$ except that all terms in $\omega_j, ~j\neq i$ are set to $0$, the term in $\cos(\omega_i t_k)$ becomes $-t_k \omega_i \sin(\omega_i t_k)$, and the term in $\sin(\omega_i t_k)$ becomes $t_k \omega_i \cos(\omega_i t_k)$.

\textbf{Derivatives with respect to $\log k_{0i}$ and $\log l_i$ :}
We recall that
\begin{flalign}
\bar{m}_{t_k}(\bm{\theta}) &=  m(t_k)+ \bm{H}_{t_k}^T\bm{F}_{t_k}\bm{m}_{t_{k-1}} \nonumber \\
\bar{v}_{t_k}(\bm{\theta}) &=  \bm{H}_{t_k}^T\bm{P}_{t_k}^-\bm{H}_{t_k} + \sigma^2 \nonumber 
\end{flalign}
with $\bm{P}_{t_k}^- = \bm{F}_{t_k} \bm{P}_{t_{k-1}}\bm{F}_{t_k}^T + \bm{K}_{t_k | t_{k-1}}$. Hence, the only terms that depend on $\log k_{0i}$ and $\log l_i$ are $\bm{F}_{t_k}$ and $\bm{K}_{t_k | t_{k-1}}$, and partial derivatives of $\bar{m}_{t_k}$ and $\bar{v}_{t_k}$ with respect to $\log k_{0i}$ and $\log l_i$  are easily derived from that of $\bm{F}_{t_k}$ and $\bm{K}_{t_k | t_{k-1}}$:
\begin{align}
&\frac{\partial \bar{m}_{t_k}(\bm{\theta})}{\partial \log k_{0i}} = \bm{H}_{t_k}^T\frac{\partial \bm{F}_{t_k}}{\partial \log k_{0i}}\bm{m}_{t_{k-1}} \nonumber \\
&\frac{\partial \bar{m}_{t_k}(\bm{\theta})}{\partial \log l_i} = \bm{H}_{t_k}^T\frac{\partial \bm{F}_{t_k}}{\partial \log l_i}\bm{m}_{t_{k-1}} \nonumber \\
&\frac{\partial \bar{v}_{t_k}(\bm{\theta})}{\partial \log k_{0i}} = \bm{H}_{t_k}^T\frac{\partial \bm{P}_{t_k}^-}{\partial \log k_{0i}}\bm{H}_{t_k} \nonumber \\
&\frac{\partial \bar{v}_{t_k}(\bm{\theta})}{\partial \log l_i} = \bm{H}_{t_k}^T\frac{\partial \bm{P}_{t_k}^-}{\partial \log l_i}\bm{H}_{t_k} \nonumber
\end{align}
with
\begin{align}
& \frac{\partial \bm{P}_{t_k}^-}{\partial \log k_{0i}} = \frac{\partial \bm{K}_{t_k | t_{k-1}}}{\partial \log k_{0i}} +\frac{\partial \bm{F}_{t_k}}{\partial \log k_{0i}} \bm{P}_{t_{k-1}}\bm{F}_{t_k}^T +  \nonumber \\
& + \bm{F}_{t_k} \bm{P}_{t_{k-1}}\frac{\partial \bm{F}_{t_k}^T}{\partial \log k_{0i}} \nonumber \\
& \frac{\partial \bm{P}_{t_k}^-}{\partial \log l_i} = \frac{\partial \bm{K}_{t_k | t_{k-1}}}{\partial \log l_i} +\frac{\partial \bm{F}_{t_k}}{\partial \log l_i} \bm{P}_{t_{k-1}}\bm{F}_{t_k}^T +  \nonumber \\
& + \bm{F}_{t_k} \bm{P}_{t_{k-1}}\frac{\partial \bm{F}_{t_k}^T}{\partial \log l_i} \nonumber.
\end{align}

\end{appendices}
\end{document}